\documentclass[journal,twoside,web]{ieeecolor}
\usepackage{generic}
\usepackage{cite}
\usepackage{amsmath,amssymb,amsfonts}
\usepackage{algorithmic}
\usepackage{graphicx}
\usepackage{algorithm,algorithmic}
\usepackage{hyperref}
\usepackage[utf8]{inputenc}
\newtheorem{thm}{Theorem}
\newtheorem{lem}{Lemma}
\newtheorem{definition}{Definition}
\newtheorem{assum}{Assumption}
\newtheorem{rem}{Remark}

\hypersetup{hidelinks=true}
\usepackage{array}
\usepackage{caption}
\usepackage{multirow}
\allowdisplaybreaks[4]
\usepackage{textcomp}
\def\BibTeX{{\rm B\kern-.05em{\sc i\kern-.025em b}\kern-.08em
    T\kern-.1667em\lower.7ex\hbox{E}\kern-.125emX}}
\begin{document}
\title{Distributed Online Bandit Nonconvex Optimization with One-Point Residual Feedback via Dynamic Regret}
\author{Youqing Hua, Shuai Liu, \IEEEmembership{Member, IEEE}, Yiguang Hong, \IEEEmembership{Fellow, IEEE},\\Karl Henrik Johansson, \IEEEmembership{Fellow, IEEE}, and Guangchen Wang, \IEEEmembership{Senior Member, IEEE}
}
\maketitle

\begin{abstract}
This paper considers the distributed online bandit optimization problem with nonconvex loss functions over a time-varying digraph. This problem can be viewed as a repeated game between a group of online players and an adversary. At each round, each player selects a decision from the constraint set, and then the adversary assigns an arbitrary, possibly nonconvex, loss function to this player. Only the loss value at the current round, rather than the entire loss function or any other information (e.g. gradient), is privately revealed to the player. Players aim to minimize a sequence of global loss functions, which are the sum of local losses.
We observe that traditional multi-point bandit algorithms are unsuitable for online optimization, where the data for the loss function are not all a priori, while the one-point bandit algorithms suffer from poor regret guarantees. To address these issues, we propose a novel one-point residual feedback distributed online algorithm. This algorithm estimates the gradient using residuals from two points, effectively reducing the regret bound while maintaining $\mathcal{O}(1)$ sampling complexity per iteration. We employ a rigorous metric, dynamic regret, to evaluate the algorithm's performance. By appropriately selecting the step size and smoothing parameters, we demonstrate that the expected dynamic regret of our algorithm is comparable to existing algorithms that use two-point feedback, provided the deviation in the objective function sequence and the path length of the minimization grows sublinearly. Finally, we validate the effectiveness of the proposed algorithm through numerical simulations.

\end{abstract}

\section{Introduction}
\label{sec:introduction}
\IEEEPARstart{O}nline optimization, as a powerful tool for sequential decision-making in dynamic environments, has experienced a resurgence in recent decades and found widespread applications in various fields such as medical diagnosis \cite{apl1}, robotics \cite{apl2}, smart grids \cite{apl3} and sensor networks \cite{apl4,apl5}. It can be understood as a repeated game between an online player and an adversary. In each round of this game, the player commits to a decision, and the adversary subsequently selects a corresponding loss function based on the player's decision. The player then incurs a loss, with partial or complete information about the loss function being revealed. The player's objective is to minimize \textit{regret}, which is defined as the gap in cumulative loss between the online player's decisions and the offline optimal decisions made with hindsight \cite{apl6,apl7,apl7+1}. An online algorithm is considered effective if the regret grows sublinearly.\par
Over the past few years, numerous centralized online optimization algorithms have been developed \cite{apl2,apl7,apl8,apl9,apl10,41flexman2005,42inprovedOP2011,43improvedOP2015,44improvedOP2015,45improvedOP2016}. 
To cater for large-scale datasets and systems, these algorithms have been adapted for distributed settings recently, considering factors like flexibility, scalability and data privacy. Prominent distributed online optimization algorithms have emerged \cite{apl1,apl4,apl5,2018duiou1,2020duiou2,2022duiou4,57Dncvx2023,2019yuanshi3,2019yuanshi4,2023yuanshi7,2023yuanshi8,2018yuanduiou1,2021yuanduiou3}, including online distributed mirror descent algorithms \cite{2018duiou1,2020duiou2,2022duiou4,57Dncvx2023}, online distributed dual averaging algorithms \cite{apl5}, online distributed gradient tracking algorithms \cite{2019yuanshi4} and online distributed primal-dual algorithms \cite{2018yuanduiou1,2021yuanduiou3}. Please refer to the survey \cite{2023survey} for the recent progress.\par
The aforementioned distributed online optimization algorithms are based on \textit{full information feedback}, implying that after a decision is made, the complete information on the current loss function is disclosed to the player. However, this requirement is often impractical in real-world applications. Obtaining complete function or gradient information will be challenging or computationally expensive in scenarios such as online advertising \cite{apl7+1}, online spam filtering \cite{47yuandeming2022}, and online source localization. Instead, only function values are accessible, rather than gradients, a situation commonly referred to as \textit{bandit feedback} in machine learning \cite{apl7+1,apl8}. In this paper, we focus on developing distributed online optimization algorithms under one-point residual bandit feedback.\par 
\subsection{Related Works}
The key step of bandit optimization  is to use gradient-free techniques to estimate the gradient of the loss function based on the zeroth-order information (i.e., function evaluations) provided by a computational oracle. 
The study of gradient-free techniques can be traced back at least to the 1960s \cite{31origins}. Recently, due to the prevalence of big data and artificial intelligence technologies, these methods have experienced a notable revival, prompting extensive research worldwide \cite{32duchi2015,33nesterov_random_2017,34ZONE2019,35wang_distributed_2019,36bergou_stochastic_2020,37primal-dual2022,38nonconvex2022,39gradient-free2022,40one-pointresidual2022,YUAN2024111328}. 
This resurgence has revitalized the field of bandit optimization, leading to the development of various novel bandit algorithms \cite{46yixinlei2021,47yuandeming2022,41flexman2005,42inprovedOP2011,43improvedOP2015,44improvedOP2015,45improvedOP2016,48YUANDOO2022,50shamir2017,52deming2021,2023yipengpang,53xinlei2023}, which can be categorized into one- and multi-point bandit feedback algorithms.\par 
One-point bandit feedback algorithms have garnered significant attention since the seminal work of Flexman et al.\cite{41flexman2005}, who introduced the one-point gradient estimator and established a regret bound of \(\mathcal{O}(dT^{3/4})\) for Lipschitz-continuous loss functions. Building upon this foundation, subsequent studies \cite{42inprovedOP2011,43improvedOP2015,44improvedOP2015,45improvedOP2016} have sought to establish smaller regret bounds under additional assumptions and specific methodologies. Recent efforts have shifted towards exploring online algorithms in distributed settings \cite{46yixinlei2021,47yuandeming2022,48YUANDOO2022}. In particular, Yi et al. \cite{46yixinlei2021} studied online bandit convex optimization with time-varying coupled inequality constraints, while Yuan et al. \cite{47yuandeming2022} focused on online optimization with long-term constraints. These works achieved an expected static regret bound of \(\mathcal{O}(d^2T^{3/4})\) for Lipschitz-continuous loss functions. Despite these advancements, a considerable gap remains between the regret guarantees achievable by traditional one-point feedback algorithms and those attainable by full-information feedback algorithms, primarily due to the large estimation variance inherent in one-point gradient estimators.\par 
Multi-point feedback models have shown promise in reducing gradient estimator variance, with two-point feedback models particularly noted for their utility and efficiency \cite{2023yipengpang, 32duchi2015, 33nesterov_random_2017, 34ZONE2019, 35wang_distributed_2019, 39gradient-free2022, 50shamir2017, 52deming2021, 53xinlei2023}. Shamir et al. \cite{50shamir2017} introduced a two-point gradient estimator that achieved expected static regret bounds of \(\mathcal{O}(d^2T^{1/2})\) for general convex loss functions and \(\mathcal{O}(d^2 \log(T))\) for strongly convex loss functions, surpassing the convergence rate of one-point feedback algorithms. Building on this, \cite{2023yipengpang} and \cite{52deming2021} developed distributed two-point algorithms with expected static regret of \(\mathcal{O}(d^2T^{1/2})\) for Lipschitz-continuous loss functions. While distributed bandit algorithms with two- or multi-point feedback often outperform one-point feedback algorithms in convergence, their reliance on multiple policy evaluations in the same environment limits their practicality. For instance, in non-stationary reinforcement learning \cite{53xinlei2023}, where the environment changes with each evaluation, multi-point algorithms lose effectiveness. Similarly, in stochastic optimization \cite{38nonconvex2022}, two-point models assume controlled data sampling, with evaluations conducted under identical conditions, which is rarely feasible in practice. Thus, there is a pressing need for advanced distributed optimization algorithms based on one-point feedback, particularly those capable of adapting to dynamic environments.\par 
In response, Zhang et al. \cite{40one-pointresidual2022} proposed a novel one-point residual feedback algorithm and demonstrated that it achieves convergence rates similar to those of gradient-based algorithms under Lipschitz-continuous functions. However, their analysis was limited to centralized static optimization scenarios. Based on the asynchronous update model, the study in \cite{D-residual} proposed a distributed zeroth-order algorithm under the same feedback model as in \cite{40one-pointresidual2022}. This approach differs significantly from the consensus-based distributed optimization algorithm studied in this paper, both in terms of algorithmic design and performance analysis. Moreover, the above work focuses primarily on offline optimization problems, leaving a gap in the application of these methods to online optimization scenarios.\par 
Research in distributed online convex optimization has been widely explored, whereas the study of nonconvex optimization remains relatively limited, largely due to the analytical challenges posed by the existence of local minima. Nevertheless, nonconvex optimization is frequently encountered in practical, real-world applications. Despite this, the literature on online nonconvex optimization is still sparse, with only a few contributions to date \cite{55Cncvx2022,54Cncvx2022,58comncvx2024,2022duiou4,57Dncvx2023}. 
In centralized settings, online nonconvex optimization methods have been developed in \cite{55Cncvx2022} and \cite{54Cncvx2022}, both achieving a gradient-size regret bound of $\mathcal{O}(T^{1/2})$. More recently, attention has shifted to distributed online nonconvex optimization \cite{57Dncvx2023,2022duiou4,58comncvx2024}. In particular, \cite{58comncvx2024} presented a distributed online two-point bandit algorithm that incorporates communication compression, achieving a sublinear regret bound, whereas \cite{58comncvx2024} proposed a online distributed mirror descent algorithm and used the first-order optimality condition  to measure the performance of the proposed algorithm.\par 
Another critical aspect of online nonconvex optimization is the choice of performance metrics. Most studies, including \cite{58comncvx2024} and \cite{2022duiou4}, have relied on static regret, which measures the difference between the total loss and the minimum loss achievable by a fixed decision over time. While static regret is a widely used performance measure \cite{54Cncvx2022,58comncvx2024,2022duiou4}, it may not fully capture the dynamic nature of real-time optimization environments, where decision variables may need to adapt to changing conditions. To address this limitation, dynamic regret has been introduced in the context of online convex optimization \cite{2019yuanshi4,2023yuanshi7,2023yipengpang,8059834,46yixinlei2021}, where the benchmark shifts to the optimal decision at each time step. This concept has recently been extended to nonconvex settings, accounting for local optimality \cite{57Dncvx2023}. Nevertheless, the theoretical analysis of dynamic regret in online nonconvex optimization, particularly under the more challenging one-point feedback scenario, remains an open problem.\par
\subsection{Main Contributions}
In this paper, we develop a novel distributed online algorithm for constrained distributed online bandit optimization (DOBO) under one-point feedback with nonconvex losses over time-varying topologies. The primary contributions of this work, compared to the existing literature, are as follows:
\begin{enumerate}
	\item 
	We propose a novel distributed online bandit algorithm with one-point residual feedback to solve constrained nonconvex DOBO problems over time-varying topologies. Our method estimates the gradient using the residuals between successive points, which significantly reduces the estimation variance compared to traditional one-point feedback algorithms \cite{41flexman2005,42inprovedOP2011,43improvedOP2015,44improvedOP2015,45improvedOP2016}. In contrast to two-point algorithms \cite{2023yipengpang, 32duchi2015, 33nesterov_random_2017, 34ZONE2019, 35wang_distributed_2019, 39gradient-free2022, 50shamir2017, 52deming2021, 53xinlei2023}, which often necessitate multiple strategy evaluations in the same environmental—an impractical assumption in many real-world applications—our approach circumvents this limitation. In other words, our proposed algorithm bridges traditional distributed online one- and two-point algorithms, providing a practical and efficient solution for distributed online optimization.
	\item The algorithm presented in this paper extends the offline centralized algorithm from \cite{40one-pointresidual2022} to an online distributed framework. Notably, this adaptation is far from trivial. The primary challenge stems from local decision inconsistencies during iterative processes, which invalidate the conventional analytical approach of constructing a perturbed recursive sequence based on gradient paradigms to derive an upper bound. As a result, the traditional framework employed in \cite{40one-pointresidual2022,D-residual} can no longer be applied, necessitating the development of novel analytical techniques tailored for distributed optimization. Furthermore, the dynamic nature of the loss function introduces additional complexity, particularly in establishing rigorous proof structures.
\item Unlike \cite{55Cncvx2022,54Cncvx2022,58comncvx2024,46yixinlei2021,47yuandeming2022,48YUANDOO2022,53xinlei2023,52deming2021}, our investigation focuses on the framework of dynamic regret, where the offline benchmark is the optimal point of the loss function at each time step. Compared to the static regret considered in \cite{54Cncvx2022,58comncvx2024,2022duiou4}, the offline benchmark for dynamic regret is more stringent. Given that the loss functions are Lipschitz-continuous, we demonstrate that our proposed algorithm achieves an expected sublinear dynamic regret bound, provided that the graph is uniformly strongly connected, the deviation in the objective function sequence and the consecutive optimal solution sequence are sublinear. To our knowledge, this work is the first to attain optimal results for online distributed bandit optimization with one-point feedback.
\end{enumerate}\par 
A detailed comparison of the algorithm proposed in this paper with related studies on online optimization in the literature is summarized in Table I.\par 

\begin{table*}
	\centering
	\caption{Comparison of This Paper to Related Works on Online Optimization}
	\label{table1}
	\renewcommand{\arraystretch}{1.8} 
	\begin{tabular}{m{2cm}<{\centering} m{2cm}<{\centering}|m{2.5cm}<{\centering}|m{3cm}<{\centering}|m{2.5cm}<{\centering}|m{3.8cm}<{\centering}}
		\hline
		Reference & Problem type & Loss functions & Feedback Model & Metric & Regret \\
		\hline
		\cite{41flexman2005} & Centralized & Convex & One-point & Static regret & $\mathcal{O}(T^{3/4})$ \\
		\hline
		\cite{47yuandeming2022},\cite{48YUANDOO2022} & Distributed & Convex & One-point & Static regret & $\mathcal{O}(T^{3/4})$ \\
		\hline
		\cite{52deming2021} & Distributed & Convex & Two-point & Static regret & $\mathcal{O}(T^{1/2})$ \\
		\hline
		\cite{2018duiou1} & Distributed & Convex & Full information & Dynamic regret & $\mathcal{O}(\omega_T^{1/2}T^{1/2})$ \\
		\hline
		\cite{2023yipengpang} & Distributed & Convex & Two-point & Dynamic regret & $\mathcal{O}(\omega_T^{1/3}T^{2/3})$ \\
		\hline
		\cite{54Cncvx2022} & Centralized & Nonconvex & Full information & Static regret & $\mathcal{O}(T^{1/2})$ \\
		\hline
		\cite{55Cncvx2022} & Centralized & Nonconvex & Two-point & Static regret & $\mathcal{O}(T^{1/2})$ \\
		\hline
		\cite{58comncvx2024} & Distributed & Nonconvex & Two-point & Static regret & $\mathcal{O}(T^{1/2})$ \\
		\hline
		\cite{2022duiou4} & Distributed & Nonconvex & Full information & Static regret & $\mathcal{O}(T^{1/2})$ \\
		\hline
		\multirow{2}{*}{\cite{57Dncvx2023}} & \multirow{2}{*}{Distributed} & \multirow{2}{*}{Nonconvex} & Full information & Dynamic regret & $\mathcal{O}((1+\Phi_T)T^{1/2}+\Psi_T)$ \\
		\cline{4-6}
		&  &  & Two-point & Dynamic regret & $\mathcal{O}((1+\Phi_T)T^{1/2}+\Psi_T)$ \\
		\hline
		This paper & Distributed & Nonconvex & One-point & Dynamic regret & $ \mathcal{O}\left( \left( 1+\Theta_{T}^2+\omega_T^2 \right)T^{1/2}\right) $ \\
		\hline		
	\end{tabular}
    \caption*{The following provides a brief explanation of the notations used in the table: \raggedright \footnotesize $\omega_T=\sum\nolimits_{k=0}^T\left\|x_{k+1}^*-x_{k}^* \right\| $; $\Phi_{T}=\sum\nolimits_{k=0}^T\sup_{x\in\Omega}|f_{k+1}(x)-f_k(x)|$; $\Psi_{T}=\sum\nolimits_{k=0}^T\sup_{x\in\Omega}\|g_{k+1}(x)-g_k(x)\|$;
    $\Theta_{T}=T\cdot \sup_{x\in\Omega,\tau\in[T]}\left|f_{i,\tau+1}(x)-f_{i,\tau}(x) \right| $.}

\end{table*}
\subsection{Outline}
The rest of the paper is organized as follows. The considered problem is formulated in Section II. In Section III, we propose the distributed online algorithm with one-point residual feedback. The dynamic regret bound of the proposed algorithm is analyzed in Section IV. Numerical experiments are presented in Section V. Finally, the paper is concluded in Section VI, with all proofs provided in the Appendix.\par
\textit{Notations}: Throughout this paper, we use $\mathbb{R}^d$ and $\mathbb{B}^d$ to denote the $d$-dimensional Euclidean space and the $d$-dimensional unit ball centered at the origin, respectively.  $\left|x \right|$ is used to represent the absolute value of scalar $x$. For any positive integer $T$, we denote set $[T]=\left\{1,2,...,T\right\}$. For vectors $x, y \in \mathbb{R}^d$, their standard inner product is $\langle x, y \rangle$, $x$'s $i$-th component is $[x]_i$, and $x^\top$ represents the transpose of the vector $x$. Write $\mathbb{E}\left[ x\right] $ to denote
the expected value of a random variable $x$. For differentiable function $f(x)$, we use $\nabla f(x)$ to represent its gradient at $x\in\mathbb{R}^d$. For a matrix  $W$, $[W]_{i,j}$ denotes the matrix entry in the $i$-th row and $j$-th column. Given a set $\Omega\in\mathbb{R}^d$ and a mapping $f:\Omega\to \mathbb{R}^d$, we call that $f$ is Lipschitz-continuous with constant $L$, if $\left\| f\left( x \right)-f\left( y \right) \right\|\le L\left\| x-y \right\|$ for any $x,y \in \mathbb{R}^d$. $\beta(k)=\mathcal{O}(\alpha(k))$ means $\limsup_{k \to \infty} \frac{\beta(k)}{\alpha(k)} < +\infty$.
\section{Problem Formulation}
\subsection{Graph Theory}
Let \(\mathcal{G}_k = (\mathcal{V}, \mathcal{E}_k, W_k)\) denote a time-varying directed graph, where \(\mathcal{V} = [n]\) denotes the set of nodes (agents), \(\mathcal{E}_k \subseteq \mathcal{V} \times \mathcal{V}\) represents the set of directed edges, and \( W_k \) is the corresponding weighted adjacency matrix. A directed edge \((j, i) \in \mathcal{E}_k\) indicates that agent \( j \) transmits information directly to agent \( i \) at time \( k \). The matrix \( W_k \) captures the communication pattern at time \( k \), with \(\left[W_k\right]_{i,j} > 0\) if \((j, i) \in \mathcal{E}_k\) and \(\left[W_k\right]_{i,j} = 0\) otherwise. Consequently, the sets of in-neighbors and out-neighbors for agent \( i \) at time \( k \) are defined as \(\mathcal{N}_{i,k}^{+} = \{ j \in \mathcal{V} \mid \left[W_k\right]_{i,j} > 0 \}\) and \(\mathcal{N}_{i,k}^{-} = \{ j \in \mathcal{V} \mid \left[W_k\right]_{j,i} > 0 \}\), respectively.\par
We make the following standard assumption on the graph.
\begin{assum}\label{assum1} 
	\ \ 
	\begin{enumerate}
		\item There exists a scalar $0<\zeta<1$ such that $\left[W_k\right]_{i,i} \ge\zeta$ for any $i\in\mathcal{V}$ and $k>0$, and $\left[W_k\right]_{i,j} \ge\zeta$ if $\left(j,i\right)\in\mathcal{E}_k$.\par 
	\item $\mathcal{G}_k$ is balanced for any $k\ge0$, and consequently, the associated weighting matrix $W_k$  is doubly stochastic, i.e., $\sum\nolimits_{j=1}^{n} \left[W_k\right]_{i,j} = \sum\nolimits_{i=1}^{n} \left[W_k\right]_{i,j} = 1$ for any $i$, $j\in\mathcal{V}$.\par  	
	\item  There exists a positive integer $ U $ such that the graph $\left(\mathcal{V},\bigcup_{t=kU+1}^{k(U+1)} \mathcal{E}_{t}\right)$ is strongly connected for any $ k \ge 0 $. \par 
	\end{enumerate}
\end{assum}\par
Assumption 1 is widely adopted in distributed optimization studies. Notably, the time-varying topologies described in Assumption 1 maintain connectivity over time but are not necessarily connected at every time instant. This makes them more general both theoretically and practically compared to the fixed connected topologies discussed in \cite{2023yuanshi8,2018duiou1,2022duiou4,34ZONE2019,37primal-dual2022,38nonconvex2022}, which can be viewed as a special case of Assumption 1 with \( U = 1 \).\par 
Next, we present a fundamental property of the matrix \(W_k\) used in this paper. Define \(W(k, s) = W_k W_{k-1} \cdots W_s\) as the transition matrices, in which $0\le k \le s $. The following lemma offers a critical result about \(W(k, s)\).
\begin{lem}[\!\cite{Nedic}]\label{lem2}
	Let Assumption \ref{assum1} holds, then for any $i, j \in \mathcal{V}$ and $0\le k \le s $, 
	\begin{equation}
		\left|\left[W\left( k,s \right)\right]_{i,j} - \frac{1}{n}\right| \leq \Gamma \gamma^{k-s}\text{,}
	\end{equation}
	where $\Gamma = (1 - {\zeta }/{4n^2})^{-2}$ and $\gamma = (1 - {\zeta }/{4n^2})^{{1}/{U}}$.
\end{lem}\par 
\subsection{Distributed Online Nonconvex Optimization}
This paper considers distributed online nonconvex optimization problems under bandit feedback. These problems can be viewed as a repeated game between $n$ online players, indexed by $i \in \left[n\right]$, and an adversary. At round $k$ of the game, each player $i$ selects a decision $x_{i,k}$ from the convex set \(\Omega \subseteq \mathbb{R}^d\), where $d$ is a positive integer. The adversary then assigns an arbitrary, possibly nonconvex, loss function $f_{i,k}$ to the player. In the bandit setting, players only observe specific loss values corresponding to their own decisions, and this information remains private. Players can communicate only with their immediate neighbors through a time-varying directed graph $\mathcal{G}_k$. The objective for the players is to collaboratively minimize the cumulative loss, defined as the sum of individual losses, while adhering to the set constraint. Specifically, at each time \(k\), the network aims to jointly solve the following nonconvex DOBO problem:
\begin{equation}\label{eq1}
	\underset{x\in \Omega }{\mathop{\min }}\,\left\{{{{f}_{k}}\left( x \right)}=\sum\limits_{i=1}^{n}{{f}_{i,k}}\left(x \right)\right\} \text{,}
\end{equation}
where $f_k$ represents the global loss function at time $k$.\par
Some basic assumptions for the Problem \eqref{eq1} are made as follows.
\begin{assum}\label{assum2}
The set $\Omega$ is compact, convex and satisfies the relation that $\left\| x-y\right\| \le D$ for any $x,y\in\Omega$. 
\end{assum} 
\begin{assum}\label{assum3}
$f_{i,k}\left(\cdot \right) $ is $L_0$-Lipschitz continuous on $\Omega$ for any $i\in\mathcal{V}$ and $k\ge0$.	
\end{assum}
\begin{assum}\label{assum4}
$f_{i,k}\left(\cdot \right) $ is $L_{1}$-smooth on $\Omega$ for any $i\in\mathcal{V}$ and $k\ge0$, i.e., $\left\|\nabla f_{i,k}\left(u\right)-\nabla f_{i,k}\left(v\right)\right\| \le L_{1}\left\|u-v\right\|$ for any $u,v\in\Omega$.	
\end{assum}\par 
It should be highlighted that no convexity assumptions are made in our analysis. Assumption 2 is common in the distributed optimization literature \cite{apl1,apl4,apl5,2018duiou1,2020duiou2,2022duiou4,57Dncvx2023,2019yuanshi3,2019yuanshi4,2023yuanshi7,2023yuanshi8,2018yuanduiou1,2021yuanduiou3}. Assumptions 3 and 4, which are standard in the distributed nonconvex optimization literature \cite{34ZONE2019,37primal-dual2022,38nonconvex2022,58comncvx2024,2022duiou4,57Dncvx2023}, are crucial for ensuring that the distributed algorithm achieves the first-order optimal point of the nonconvex optimization problem at an optimal convergence rate.

\subsection{Gaussian Smoothing}
The core of online bandit (or offline zeroth-order) optimization is to estimate the gradient of a function using a zeroth-order oracle. This motivates the exploration of methods for gradient estimation via Gaussian smoothing, pioneered by Nesterov and Spokoiny \cite{33nesterov_random_2017}. This approach, which constructs gradient approximations solely from function values, has gained popularity and is featured in several recent papers, including \cite{34ZONE2019,37primal-dual2022,38nonconvex2022,39gradient-free2022,2023yipengpang,53xinlei2023}.\par 
The Gaussian smoothing version of a given function $f$ is performed as $f^s\left(x \right) =\mathbb{E}_{u\sim \mathcal{N}\left( 0,1\right) }\left(x+\mu u \right)$, where $\mu >0$ is an smoothing parameter. As proved in \cite{33nesterov_random_2017}, the following lemma provides a fundamental property of the smoothed function \( f^s \) that is crucial for our subsequent analysis.
\begin{lem}\label{lem21}
	Consider $f$ and its Gaussian-smoothed version $f_s$. For any $x\in\Omega$, $f^{s}(x)$ is approximated with an error bounded as:
	\begin{align}
		|f^{s}(x) - f(x)| \leq \begin{cases}
			\mu L_0 \sqrt{d}, & \text{if } f \in \mathcal{C}^{0,0} \\
			\mu^{2} L_1 d, & \text{if } f \in \mathcal{C}^{1,1}.
		\end{cases}
	\end{align}\par 
\end{lem}

\subsection{Dynamic Regret}
For online convex optimization, the standard performance metric is \textit{regret}, which is defined as the gap between the accumulated loss of the decision sequence $\left\{x_{j,k}\right\}_{k=1}^{T}$ and the expert benchmark $\left\{{{y}_{k}}\right\}_{k=1}^{T}$, formally expressed as:
\begin{equation}\label{eq3}
	\mathcal{R}_{j,T}^\text{cv}=\sum\nolimits_{k=1}^{T}{\sum\nolimits_{i=1}^{n}{{{f}_{i,k}}({{x}_{j,k}})}}-\sum\nolimits_{k=1}^{T}{{{f}_{k}}({{y}_{k}})}\text{.}
\end{equation}
Note that for general nonconvex optimization problems, finding a global minimum is NP-hard even in the centralized setting \cite{10.5555/3305381.3305529}. Directly extending \eqref{eq3} to nonconvex optimization leads to intractable bounds. Therefore, the primary goal in online nonconvex optimization is to develop algorithms that converge to a set of stationary points. In the following, we will introduce the concept of regret for online nonconvex optimization. Before proceeding, it is essential to revisit the definition of a stationary point.\par 
\begin{definition}[\!\cite{next}]
	For a convex set \(\Omega \subset \mathbb{R}^d\) and a function \(f: \Omega \rightarrow \mathbb{R}\), if the condition \(\langle \nabla f(x^*), x - x^* \rangle \leq 0\) holds for all \(x \in \Omega\), then the point \(x^*\) is defined as a stationary point of the optimization problem $\mathop{\min}_{x\in\Omega}f\left( x\right)$.
\end{definition}\par  
Based on \eqref{eq3} and the definition of stationary points, the regret for online nonconvex optimization \eqref{eq1} can be formulated as follows:
\begin{align}
	&\mathcal{R}_{j,T}^\text{ncv}=\sum\nolimits_{k=1}^{T}\left\langle \nabla f_k\left( x_{j,k}\right),x_{j,k} \right\rangle -\sum\nolimits_{k=1}^{T}\left\langle  \nabla f_k\left( x_{j,k}\right),y_k \right\rangle \text{.}\nonumber
\end{align}
The choice of different benchmarks will result in different categories of regret. Specifically, when considering a dynamic benchmark \(x_{k}^{*}\) that satisfies \(\langle \nabla f_k(x_{k}^{*}), x_{k}^{*} - x \rangle \leq 0\) for all \(x \in \Omega\), the associated regret is termed \textit{dynamic regret}, defined as
\begin{align}\label{eq4}
	&\mathcal{D}\mathcal{R}_{j,T}^\text{ncv}\hspace{-2pt}=\hspace{-2pt}\sum\limits_{k=1}^{T}\hspace{-2pt}\Big\lbrace \left\langle \nabla f_k\left( x_{j,k}\right),x_{j,k} \right\rangle-\inf\nolimits_{x_k\in\Omega}\left\langle f_k\left( x_{j,k}\right),x \right\rangle \Big\rbrace   \text{.}
\end{align}
Similarly, when employing a static benchmark $x^{*}$ that satisfies $ \left\langle \sum\nolimits_{k=1}^T\nabla f_k \left( x^{*}\right),x^{*}-x \right\rangle \le 0$ for all $x\in \Omega$, the associated regret is referred to as \textit{static regret} and given by
\begin{align}\label{eq5}
	\mathcal{S}\mathcal{R}_{j,T}^\text{ncv}&\hspace{-2pt}=\hspace{-2pt}\sum\limits_{k=1}^{T}\hspace{-2pt}\left\langle \nabla f_k\left( x_{j,k}\right),x_{j,k} \right\rangle-\inf_{x\in\Omega}\left\langle \sum\nolimits_{k=1}^{T} f_k\left( x_{j,k}\right),x \right\rangle \text{.}
\end{align}
Notably, the benchmark for dynamic regret involves identifying a stationary point of the objective function at each time \( k \), whereas the benchmark for static regret involves identifying a stationary point of the optimization problem \(\min_{x\in\Omega}\sum\nolimits_{k=1}^T f_k(x)\), which remains time-invariant throughout the time horizon. It is evident that dynamic regret \eqref{eq4} is more stringent than static regret \eqref{eq5}.\par 
In this paper, we consider the dynamic regret of Problem \eqref{eq1} under the bandit feedback setting. Due to the inherent stochasticity of algorithms in this framework, we focus on the average version of the dynamic regret
\begin{align}\label{eq6}
	\mathbb{E}-\mathcal{D}\mathcal{R}_{j,T}^\text{ncv}&=\sum\nolimits_{k=1}^{T}\Big\lbrace \mathbb{E}\left[ \left\langle \nabla f_k\left( x_{j,k}\right),x_{j,k} \right\rangle\right]  \nonumber\\
	&\quad-\inf\nolimits_{x_k\in\Omega}\mathbb{E}\left[ \left\langle f_k\left( x_{j,k}\right),x \right\rangle\right]  \Big\rbrace   \text{.}
\end{align} 
Dynamic regret is known to render problems intractable in the worst-case scenario. Drawing from \cite{57Dncvx2023}, \cite{2023yipengpang} and \cite{10.5555/3305381.3305529}, we characterize the difficulty of the problem by using the deviation in the objective function sequence
\begin{equation}\label{theta_T}
	\theta_{i,k}={\sup}_{x\in\Omega,\tau\in[k]}\left|f_{i,\tau+1}(x)-f_{i,\tau}(x) \right| \text{,}
\end{equation}
\begin{equation}\label{Theta_T}
	\Theta_T=T\sum\nolimits_{i=1}^n{\theta_{i,T}} \text{,}
\end{equation}
and the minimizer path length (i.e., the deviation of the consecutive optimal solution sequence)
\begin{equation}\label{omega_k}
	\omega_{i,k}= {\left\|x_{i,k+1}^*-x_{i,k}^* \right\|} \text{,}
\end{equation}
\begin{equation}\label{omega_T}
	\omega_T=\sum\nolimits_{k=1}^T\sum\nolimits_{i=1}^{n}{ {{\left\|x_{i,k+1}^*-x_{i,k}^* \right\|}}}  \text{,}
\end{equation}
where $x_{i,k}^*=\min_{x\in\Omega}f_{i,k}\left(x \right) $. The primary objective of this work is to develop an online distributed optimization algorithm to solve problem \eqref{eq1}, ensuring that the dynamic regret \eqref{eq6} grows sublinearly, provided that the growth rates of $\Theta_T$ and $\omega_T$ remain within a certain range.

\section{OP-DOPGD Algotithm}
In this section, we present a distributed online optimization algorithm with one-point residual feedback to address Problem \eqref{eq1}. The efficacy of the proposed algorithm is evaluated using the expected dynamic regret \eqref{eq6}.\par 
To proceed, we first introduce a classic algorithm commonly used for distributed online constraint optimization with full information feedback: the distributed projected gradient descent algorithm\cite{Nedic}, which is given as follows:
\begin{align}
	&y_{i,k}=x_{i,k}-\alpha_k\nabla f_{i,k}\left(x_{i,k}\right) \text{,}\label{5}\\
	&x_{i,k+1}=\mathcal{P}_{\Omega }\left[ \sum\nolimits_{j=1}^{n}{\left[W_k\right]_{i,j}}{{y}_{j,k}} \right] \text{.} \label{6}
\end{align}
Here, $x_{i,k}$ represents the decision by agent $i$ at step $k$, and $\alpha_k > 0$ is a non-increasing step size.\par 
\begin{algorithm}[H]
	\caption {Distributed online projected gradient descent algorithm with one-point residual feedback (OP-DOPGD)}\label{alg:alg1}
	\begin{algorithmic}
		\STATE 
		\STATE \textbf{Input}: non-increasing and positive sequences $\left\{\alpha_k\right\}$, $\left\{\mu_k\right\}$.
		\STATE \textbf{Initialize}: $x_{i,0}\in\Omega$, for all $i\in \mathcal{V}$.
		\STATE  \textbf{for} $t=1$ to $T$ \textbf{do}  \par
		\STATE  \quad\textbf{for} $i=1$ to $n$ in parallel \textbf{do}  \par
		\hspace{0.5cm} Select vector $u_{i,k}\sim \mathcal{N}\left( 0,1 \right)$ independently and randomly.\par 
		\hspace{0.5cm} Query $f_{i,k}\left(x_{i,k}+\mu_ku_{i,k}\right)$ and receive $g_{i,k}^s$ by \eqref{eq11}.\par 
		\hspace{0.5cm} Update
		\begin{equation}\label{eq12}
			y_{i,k}=x_{i,k}-\alpha_{k}g_{i,k}^{s}({{x}_{i,k}})\text{,}
		\end{equation}\par
		\begin{equation}\label{eq13}
			{{x}_{i,k+1}}=\mathcal{P}_{\Omega }\left[ \sum\nolimits_{j=1}^{n}{\left[W_k\right]_{i,j}}{{y}_{j,k}} \right]\text{.}
		\end{equation}
		\STATE \quad\textbf{end for}
		\STATE \textbf{end for}
	\end{algorithmic}
	\label{alg1}
\end{algorithm}
Building on this foundation, various algorithms have been developed to solve problem \eqref{eq1} in the bandit feedback setting. However, the conventional one-point bandit algorithms employed in [2], [4], and [6] exhibit poor regret guarantees. Furthermore, the two-point gradient estimators used in \cite{2023yipengpang, 32duchi2015,33nesterov_random_2017,34ZONE2019,35wang_distributed_2019,39gradient-free2022,50shamir2017,52deming2021,53xinlei2023} are observed to be unpractical for online optimization, where data are not all available a priori. These challenges motivate our research on enhanced distributed online algorithms with $\mathcal{O}(1)$ sampling complexity per iteration and improved regret guarantees.\par 
In this paper, we develop a distributed online optimization algorithm based on the following one-point residual feedback model:
\begin{align}\label{eq11}
	g_{i,k}^s\left(x_{i,k}\right)&=\frac{u_{i,k}}{\mu_k}\big(f_{i,k}\left(x_{i,k}+\mu_ku_{i,k}\right)\nonumber\\
	&\quad-f_{i,k-1}\left(x_{i,k-1}+\mu_{k-1}u_{i,k-1}\right)\big)\text{,}
\end{align}
where $u_{i,k}$ and $u_{i,k-1}$ are independent random vectors sampled from the standard multivariate Gaussian distribution, $\mu_k>0$ is a non-decaying exploration parameter. It can be observed that the gradient estimate in \eqref{eq11} evaluates the loss value at only one perturbed point $x_{i,k}+\mu_ku_{i,k}$ at each iteration \( k \), while the other loss evaluation \( f_{i,k-1}\left(x_{i,k-1}+\mu_{k-1}u_{i,k-1}\right) \) is inherited from the previous iteration. This constitutes a one-point feedback scheme based on the residual between two consecutive feedback points, referred to as \textit{one-point residual feedback} in \cite{40one-pointresidual2022}. Combining \eqref{5}, \eqref{6} with the one-point gradient estimator \eqref{eq11}, our algorithm for solving Problem \eqref{eq1} is outlined in pseudocode as Algorithm 1. \par
To implement, in each round $k$, each player $i$ generates a gradient estimate of the current local loss function based on \eqref{eq11}. Subsequently, the player performs gradient descent to obtain the intermediate variable $y_{i,k}$, as shown in \eqref{eq12}. In the distributed setting, player $i$ is only allowed to communicate with its instant neighbors through a time-varying digraph $\mathcal{G}_k$. Using the information received from these neighbors, player $i$ applies the projected consensus-based algorithm to update its decision to $x_{i,k+1}$ in \eqref{eq13}.

\begin{rem}

It is crucial to clarify the connection and distinction between the one-point residual feedback model \eqref{eq11} and the commonly used two-point gradient estimators \cite{2023yipengpang, 32duchi2015, 33nesterov_random_2017, 34ZONE2019, 35wang_distributed_2019, 39gradient-free2022, 50shamir2017, 52deming2021, 53xinlei2023}. Both models utilize residuals between two random points to estimate the gradient, but they differ significantly in implementation. The two-point method requires the evaluation of function values at two random points per iteration, which is generally computationally expensive. In contrast, \eqref{eq11} only requires one function evaluation, inheriting the value from the previous iteration for the second point. It is observed that one-point residual feedback offers a more practical alternative, particularly in online optimization, where the data for the loss function are not all available a priori. A key limitation of the two-point method is its dependence on performing two different policy evaluations within the same environment—often an impractical requirement in dynamic settings. For instance, in non-stationary reinforcement learning scenarios, the environment undergoes changes after each policy evaluation, rendering the two-point approach inapplicable. Conversely, the residual feedback mechanism in \eqref{eq11} circumvents this issue by computing the residual between two consecutive feedback points. Consequently, our investigation focuses on the zeroth-order algorithm with one-point residual feedback for distributed online optimization.

\end{rem}
\begin{rem}
Algorithm 1 employs a consensus-based strategy \eqref{eq13}, extending the one-point residual feedback model for centralized optimization \cite{40one-pointresidual2022} to a distributed setting. It also adapts this static model for online optimization by incorporating the dynamic nature of the loss functions into the framework. This integration poses significant challenges for analysis due to the inherent variability and unpredictability of time-varying functions. 
Algorithm 1, which is designed for distributed online nonconvex optimization over time-varying directed topologies, differs from previous studies on fixed topologies \cite{2023yuanshi8,2018duiou1,2022duiou4,34ZONE2019,37primal-dual2022,38nonconvex2022} or convex optimization problems \cite{2019yuanshi4,2023yuanshi7,2023yipengpang,8059834,46yixinlei2021}. To our knowledge, this is the first study for distributed online non-convex optimization with one-point feedback. Furthermore, our study employs time-varying exploration parameters $\mu_k$, offering greater flexibility compared to the fixed values used in \cite{2023yipengpang, 32duchi2015, 33nesterov_random_2017, 34ZONE2019, 35wang_distributed_2019, 39gradient-free2022, 50shamir2017, 52deming2021, 53xinlei2023}.
\end{rem}

\section{Performance Analysis via Dynamic Regret}
This section focuses on demonstrating the convergence of Algorithm 1 applied to Problem \eqref{eq1} by providing upper bounds on its dynamic regret. For the sake of clarity, we first present a few fundamental properties of the gradient estimator \eqref{eq11}, which are essential for the subsequent analysis. Following this, the primary convergence results of the OP-DOPGD algorithm will be presented.
\subsection{Properties of the Gradient Estimator}
This subsection presents several key properties of the gradient estimator \eqref{eq11} within the context of the distributed projected gradient descent algorithm. It is crucial to highlight, as noted in \cite{32duchi2015}, that a defining characteristic of zeroth-order methods is that the gradient estimator is nearly unbiased and possesses a small norm. Therefore, we begin by demonstrating that the gradient estimator \eqref{eq11} provides an unbiased estimate of the smoothed function $f_{i,k}^{s}$.
\begin{lem}\label{lem3}
	If $g_{i,k}^s$ is calculated by \eqref{eq11}, then for any $x_{i,k}\in\Omega$, $i \in \mathcal{V}$ and $k\ge0$, we have $\mathbb{E}\big[ g_{i,k}^{s}(x_{i,k}) \big]=\nabla f_{i,k}^{s}(x_{i,k})$.
\end{lem}\par 
\begin{proof}
The proof of Lemma \ref{lem3} is straightforward. Given that \( u_k \) is independent of \( u_{k-1} \) and has zero mean, by considering the expression for $g_{i,k}^s$, the conclusion follows. 
\end{proof}\par 
Next, we establish an upper bound on the expected norm of the gradient estimate.
\begin{lem}\label{lem31}
	Suppose that $f_{i,k}$ is $L_0$-Lipschitz continuous on $\Omega$, $g_{i,k}^s$ is computed using \eqref{eq11}, and ${x_{i,k}}$ is generated by Algorithm 1. Then, for each $i \in \mathcal{V}$ and $k \ge 0$, the following inequality holds:
	\begin{align}\label{lem41}
		& \mathbb{E}\left[ \left\| g_{i,k}^{s}({{x}_{i,k}}) \right\| \right]\nonumber\\
		&\le \frac{\sqrt{3d}{{L}_{0}}{{\alpha }_{k-1}}}{{{\mu }_{k}}}\sum\limits_{j=1}^{n}\left[W_{k-1}\right]_{i,j}\mathbb{E}\left[ \left\| g_{j,k-1}^{s}({{x}_{j,k-1}}) \right\| \right] \nonumber\\ 
		&\quad+\frac{\sqrt{3d}{{L}_{0}}}{{{\mu }_{k}}}\sum\limits_{j=1}^{n}{\left[W_{k-1}\right]_{i,j}}\mathbb{E}\left[ \left\| {{x}_{j,k-1}}-{{x}_{i,k-1}} \right\| \right] \nonumber\\ 
		&\quad+\frac{2\sqrt{3}\left( d+4 \right){{L}_{0}}{{\mu }_{k-1}}}{{{\mu }_{k}}}+\frac{\sqrt{3d}}{{{\mu }_{k}}}{{\theta }_{i,k}} \text{,}
	\end{align}
	where ${\theta }_{i,k}$ is defined in \eqref{theta_T}.
\end{lem}\par

\begin{proof}
The proof is provided in Appendix A. 
\end{proof} 
\begin{rem}
It is important to note that while \( g_{i,k}^s \) serves as an unbiased estimator of the smoothing function \( f_{i,k}^s \), the difference between \( f_{i,k}^s \) and the original loss function \( f_{i,k} \), as elucidated in Lemma \ref{lem21}, introduces a bias in the gradient estimate derived from the \eqref{eq11} estimator. This bias introduces significant complexities in the proof structure. Furthermore, Lemma \ref{lem31} demonstrates that the gradient estimator \eqref{eq11} undergoes a contraction under the update rules of Algorithm 1, with a contraction factor expressed as \(\mathcal{N}_k=\frac{\sqrt{3d}{{L}_{0}}{{\alpha }_{k-1}}}{{{\mu }_{k}}}\). This finding extends the results of \cite{40one-pointresidual2022} to the domain of distributed online optimization. A notable difference is the communication between local variables of agents in the distributed setting, which introduces an additional penalty due to different decisions made by nodes in the network. Additionally, further perturbations arise from the dynamic nature of the loss function. These two perturbations pose significant challenges to the algorithmic analysis and distinguish our analytical framework from that in \cite{40one-pointresidual2022}.
\end{rem}

\subsection{Nonconvex case}
This subsection presents the main convergence results of the OP-DOPGD algorithm as applied to Problem \eqref{eq1}. We use the metric of dynamic regret to evaluate the algorithm's performance.\par  
We begin by deriving an upper bound for the time-averaged consensus error among the agents. To facilitate this, we define the average state of all agents at step \( k \) as follows: \(\bar{x}_k=\frac{1}{n}\sum\nolimits_{i=1}^n x_{i,k}\).
\begin{lem}\label{lem4}
	Given that Assumptions 1--3 hold, suppose \( g_{i,k}^s \) is computed using \eqref{eq11}, and \( x_{i,k} \) is updated according to Algorithm 1, with the step-size \( \alpha_k = \frac{1}{2Mk^a} \) and the smoothing parameter \( \mu_k = \frac{1}{k^b} \), where \( 0 < b \leq a < 1 \) and \( M = \frac{4\sqrt{3d}n\Gamma L_{0}\gamma}{1-\gamma} \) with $0<\gamma<1$ as defined in Lemma 1. Then, for all \( k \in [T] \) and \( i \in \mathcal{V} \), we have
	
	\begin{align}\label{lem51}
		\sum_{i=1}^{n}\mathbb{E}\left[\left\|g_{i,k}^s\right\|\right] = \mathcal{O}(1) + \mathcal{O}\left( k^b \sum_{i=1}^{n}{\tilde{\Theta}_{i,k}} \right),
	\end{align}
	and
	\begin{align}\label{lem52}
		\sum_{k=1}^{T} \sum_{i=1}^{n} \mathbb{E}\left[\left\| x_{i,k} - \bar{x}_k \right\|\right] = \mathcal{O}\left( T^{1-a} \right) + \mathcal{O}\left( T^{b-a} \Theta_T \right)\text{,}
	\end{align}
	where \( \tilde{\Theta}_{i,k} = \max_{\tau \in [k]} \theta_{i,\tau} \), with \( \theta_{i,\tau} \) and \( \Theta_T \) defined in \eqref{theta_T} and \eqref{Theta_T}, respectively.
\end{lem}\par 
\begin{proof}
The explicit expressions on the right-hand side of \eqref{lem51} and \eqref{lem52}, along with the details of the proof, are provided in Appendix B.
\end{proof}\par
\begin{rem}
The original intention of designing the one-point residual feedback model \eqref{eq11} was to reduce the large variance caused by traditional one-point gradient estimation, thereby achieving a better regret guarantee while avoiding the higher query complexity associated with multi-point gradient estimation. The upper bound on the norm of the gradient for traditional one-point gradient estimation is on the order of $\mathcal{O}(\mu_k^{-1})$. As demonstrated in \eqref{lem51}, our approach achieves a significantly lower bound. Note that if the deviation in the sequence of loss functions $\sum\nolimits_{i=1}^{n}{\tilde{\Theta}_{i,k}}$ is known a priori and grows slower than $\mathcal{O}(k^{-b})$, then the norm can be bounded by a constant, analogous to the results of two-point gradient estimation. Additionally, in the distributed online setting, our results require a distinct analytical framework due to the increased complexity compared to centralized static optimization, as discussed in [40]. The primary difference lies in the interaction between the consensus error among distributed agents and the expected norm of the gradient estimation. Grasping this intrinsic connection and establishing tight upper bounds for the gradient is a significant challenge.
\end{rem}

Now we are ready to establish a bound for the expected dynamic regret $\mathbb{E}-\mathcal{D}\mathcal{R}_{j,T}^\text{ncv}$ of Algorithm 1 for the distributed online nonconvex optimization Problem \eqref{eq1}.\par 

\begin{thm}
Consider the constrained DOBO problem \eqref{eq1} under Assumptions 1--4, with nonconvex loss functions. Let the decision sequences $\left\{x_{i,k}\right\}^{T}_{k=1}$ and  $\left\{y_{i,k}\right\}^{T}_{k=1}$ be generated by Algorithm 1, where the step size \(\alpha_k\) and smoothing parameter \(\mu_k\) are defined as
\[ \alpha_k = \frac{1}{2Mk^a} \quad \text{and} \quad \mu_k = \frac{1}{k^b}\text{,} \]
with parameters $a, b \in (0,1)$ satisfy $0 \le \delta \le a - b $, and the constant $M = \frac{4\sqrt{3d}n\Gamma L_{0}\gamma}{1-\gamma}$ with $0<\gamma<1$ as defined in Lemma 1.
Then, for any $T\ge1$ and $j\in \mathcal{V}$, the resulting dynamic regret satisfies:
\begin{align}\label{thm1}
	\mathbb{E}-\mathcal{D}\mathcal{R}_{j,T}^\text{ncv}\le\mathcal{O}\bigg(&\max\bigg\{ T^{\max\left\{a,1-b\right\}},T^{2b-a}\Theta_{T}^2,\nonumber\\
	&T^{b-a}\Theta_{T} ,T^{a}\omega_T^2\bigg\}\bigg).
\end{align} 
Moreover, setting \(a = \frac{1}{2}\) yields an improved bound on the dynamic regret:
\begin{align}\label{thm2}
	\mathbb{E}-\mathcal{D}\mathcal{R}_{j,T}^\text{ncv}\le \mathcal{O}\left(\max\left\{T^{1/2 + \delta},T^{1/2-2\delta}\Theta_{T}^2,T^{1/2}\omega_T^2 \right\}\right)
\end{align}
 for any $\delta \in [0, \frac{1}{2})$.  
In the specific case where $a = b $, the dynamic regret simplifies to
$ \mathcal{O}\left(\max\left\{T^{1/2},T^{1/2}\Theta_{T}^2,T^{1/2}\omega_T^2 \right\}\right)$. Here, \( \tilde{\Theta}_{i,k} = \max_{\tau \in [k]} \theta_{i,\tau} \), and  \( \theta_{i,\tau} \), \( \Theta_T \) and \( \omega_T \) are defined in \eqref{theta_T}, \eqref{Theta_T} and \eqref{omega_T}, respectively.
\end{thm} 
	\begin{proof}
The explicit expressions on the right-hand side of \eqref{thm1} and \eqref{thm2}, along with the details of the proof, are provided in Appendix C.
	\end{proof}
\begin{rem}
Theorem 1 shows that Algorithm 1 achieves improved performance compared to the dynamic regret bound $\mathcal{O}(\omega_T^{1/3}T^{2/3})$ established by the distributed online optimization algorithm for convex optimization in \cite{2023yipengpang}. Furthermore, \eqref{thm2} shows that Algorithm 1 recovers the regret bound of $\mathcal{O}(T^{1/2+c})$, where $c\in (0,{1}/{2})$, established by the online optimization algorithm under full information feedback in \cite{8059834}, even though \cite{8059834} uses the standard static regret metric rather than the stricter dynamic metric. However, it is important to note that the bounds in Theorem 1 are slightly worse than the $\mathcal{O}(T^{1/2})$ static regret bound of the centralized online algorithms described in \cite{54Cncvx2022,55Cncvx2022}. This small difference is justified because these algorithms make trade-offs in query complexity and are centralized. Algorithm 1 is more suitable for online optimization where the data is not known a priori, offering a practical balance between performance and complexity.
\end{rem}
\begin{rem}
It should be emphasized that when \(\Theta_T = 0\) and \(\omega_T = 0\), the problem described by \eqref{eq1} simplifies to a static distributed optimization problem. In this case, Theorem 1 provides an optimization error bound of \(\mathcal{O}(d^2T^{\max\left\lbrace a,1-b\right\rbrace })\) with \(0 < b \le a < 1\), derived from the one-point residual feedback model under static distributed optimization. This bound surpasses the one obtained by the one-point algorithm \cite{YUAN2024111328} in distributed convex optimization, \(\mathcal{O}(d^2T^{-1/3})\). Moreover, when \(a = b = 1/2\), it reaches the bound of \(\mathcal{O}(d^2T^{-1/2})\) achieved by the two-point zeroth-order algorithms \cite{37primal-dual2022,38nonconvex2022}, marking a significant and unexpected improvement, largely due to the low query complexity and the inherent distributed characteristics of the model. To our knowledge, this is the first result for one-point zeroth-order algorithms solving distributed nonconvex optimization.

\end{rem}
\subsection{Convex case}
In this subsection, we consider convex loss functions. For the constrained convex DOBO problem \eqref{eq1}, we introduce the expected dynamic network regret \cite{2023yipengpang} for an arbitrary node \( j \) to evaluate Algorithm 1's performance:
\begin{equation*}
	\mathbb{E}-\mathcal{DR}_{j,T}^\text{cv} = \sum\nolimits_{k=1}^{T} \mathbb{E}\left[ f_{k}(x_{j,k}) \right] - \sum\nolimits_{k=1}^{T} \mathbb{E}\left[f_{k}(x_{k}^*)\right] \text{,}
\end{equation*}
where \( x_{k}^* = \arg\min_{x \in \Omega} f_k(x) \). This measures the cumulative loss discrepancy between the decisions made by agent \( j \) and the optimal solutions.\par 
Next, we provide an upper bound for the expected dynamic regret of Algorithm 1 for the constrained convex DOBO Problem \eqref{eq1}.
\begin{thm}
	Under Assumptions 1--3, we consider the constrained DOBO Problem \eqref{eq1} with convex losses. Let the decision sequences $\left\{x_{i,k}\right\}^{T}_{k=1}$ and  $\left\{y_{i,k}\right\}^{T}_{k=1}$ be generated by Algorithm 1 and take $\alpha_k=\frac{1}{2M\sqrt{k}}$, $\mu_k=\frac{1}{\sqrt{k}}$ with the constant $M = \frac{4\sqrt{3d}n\Gamma L_{0}\gamma}{1-\gamma}$ for all $k>0$, where $0<\gamma<1$ is defined in Lemma 1. Then, for all $T\ge1$ and $j\in \mathcal{V}$, the resulting dynamic regret satisfies
	\begin{align}\label{convex}
		&\mathbb{E}-\mathcal{D}\mathcal{R}_{j,T}^\text{cv}\le\left\lbrace  {nD^{2}M}+2nD{{\left( d+3 \right)}^{3/2}}+\frac{\mathcal{H}_1}{2M}\right\rbrace\sqrt{T} \nonumber\\
		&+{nM\sqrt{T}\omega_T^2}+{2nDM\sqrt{T}\omega_T}+\frac{\mathcal{H}_2\sqrt{T}}{2M}{\Theta}_{T}^2+ \frac{\mathcal{H}_3}{2M} \Theta_T\text{,}
	\end{align}
	where $\mathcal{H}_1=48\left( d+4\right)^2+\frac{\sqrt{3}(32+24n\Gamma )\left( d+4\right)L_0 }{1-\gamma}$, $\mathcal{H}_2=48nd$, $\mathcal{H}_3=\frac{(32+24n\Gamma )L_0\sqrt{3d}}{1-\gamma}$, $\tilde{\Theta}_{i,k}=\max_{\tau\in[k]}\theta_{i,k}$, and $\theta_{i,k}$, $\omega_{T}$ are defined in \eqref{theta_T} and \eqref{omega_k}, respectively. 
\end{thm}
\begin{proof}
	The explicit expressions on the right-hand side of \eqref{convex} and the details of the proof are provided in Appendix D.
\end{proof}
\begin{rem}
	First, the dimensional dependence of the proposed method is \(\mathcal{O}(d^2)\), which is common for distributed zeroth-order algorithms and consistent with the results in \cite{52deming2021,35wang_distributed_2019,37primal-dual2022,38nonconvex2022}. Dimensional dependence is a crucial performance metric as it directly impacts the scalability of the algorithm in high-dimensional settings. Recently, significant work has focused on analyzing the dimensional dependence of various zeroth-order methods. In \cite{32duchi2015}, Duchi et al. demonstrated that the lower bounds on the convergence rate of zeroth-order stochastic approximation can be \(\mathcal{O}(d)\) in smooth cases and \(\mathcal{O}(d \ln d)\) in non-smooth cases. Second, Algorithm 1 achieves a regret bound of \(\mathcal{O}({T}^{1/2})\) in distributed online bandit convex optimization, which matches the regret guarantee of the two-point mirror descent algorithm in \cite{52deming2021}. However, unlike the two-point methods, Algorithm 1 requires only a single function query per iteration, making it more practical in online optimization.
\end{rem}

\section{Simulation}
We evaluate the performance of the proposed algorithm through numerical simulations. Specifically, Algorithm 1 is applied to both convex and nonconvex DOBO problems, with a focus on analyzing its dynamic regret bounds.
\subsection{Convex Case}
In this subsection, we evaluate the performance of Algorithm 1 in solving a distributed dynamic tracking problem. We consider a sensor network consisting of ten sensors, labeled as $\mathcal{V}=\{1,\dots,10\}$. Each sensor communicates with its neighbors through a time-varying communication topology, which can take one of the four possible configurations shown in Fig. 1. For each time step $k$, the weighted matrix $W_k$ is defined such that $\left[W_k\right]_{i,j} = \frac{1}{|\mathcal{N}_{i,k}^{+}|}$ if $j$ is an in-neighbor of $i$, where $|\mathcal{N}_{i,k}^{+}|$ denotes the number of in-neighbors of sensor $i$ at time $k$. Notably, the union of the four possible graphs in Fig. 1 forms a strongly connected graph. The four graphs switch periodically with a period of $B=4$.

\begin{figure}
	\centering
	\includegraphics[width=0.85\linewidth]{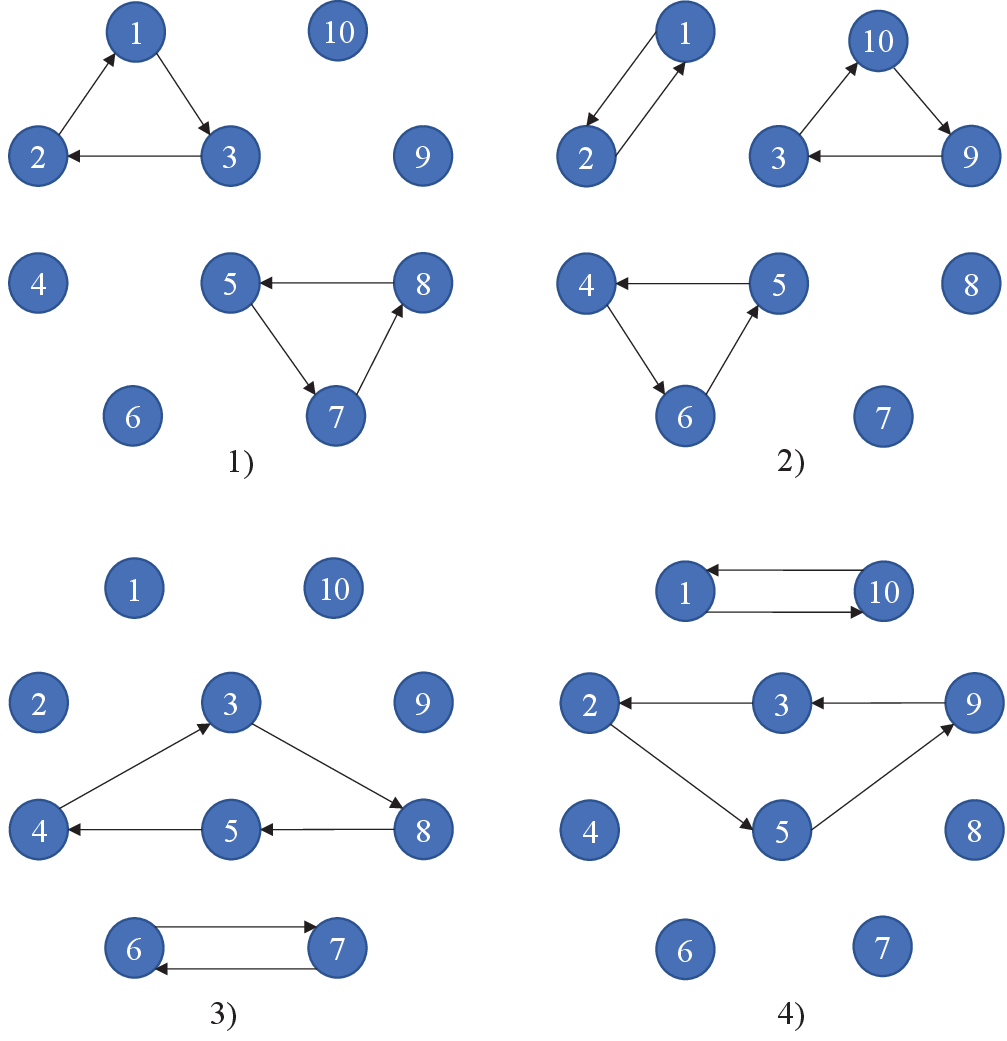}
	\caption{Time-varying graph configurations}
	\label{fig:haha}
\end{figure}

We consider a slowly moving target in a 2-D plane. The target's position at each time $k$ is denoted by $x_k^*$, and it evolves dynamically according to the following equation \cite{57Dncvx2023}:
\[
x_{k+1}^* = x_k^* + \begin{bmatrix}
	\frac{(-1)^{q_k} \sin(k/50)}{10k} \\
	\frac{-q_k \cos(k/70)}{40k}
\end{bmatrix}\text{,}
\]
where $q_k \sim \text{Bernoulli}(0.5)$, and the initial position is $x_0^* = [0.8, 0.95]^\top$.

At time $k$, each sensor $i$ observes the distance measurement $z_{i,k}$ between its position $s_i$ and the target position $x_k^*$, given by $z_{i,k} = \|x_k^* - s_i\|^2$. The positions of the sensors are: $s_1 = [1, 3]^\top$, $s_2 = [2, 5]^\top$, $s_3 = [5, 1]^\top$, $s_4 = [2, 4]^\top$, $s_5 = [3, 1]^\top$, $s_6 = [2, 3]^\top$, $s_7 = [2, 6]^\top$, $s_8 = [4, 2]^\top$, $s_9 = [1, 2]^\top$, $s_{10} = [1, 1]^\top$. The local square loss function for each sensor $i$ is defined as
\[
f_{i,k}(x) = \frac{1}{4}\left(\|x - s_i\|^2 - z_{i,k}\right)^2.
\]
The sensors collaboratively solve the following optimization problem with a quartic objective function:
\[
\min_{x \in \Omega} f_k(x) = \frac{1}{10} \sum_{i=1}^{10} f_{i,k}(x),
\]
where $\Omega \subset \mathbb{R}^2$ is a compact and convex set representing the geographical boundary of the target's position, defined as $\Omega = \{x \in \mathbb{R}^2 \mid \|x_1\| + \|x_2\| \leq 3\}$.\par 
\begin{figure}
	\centering
	\includegraphics[width=1\linewidth]{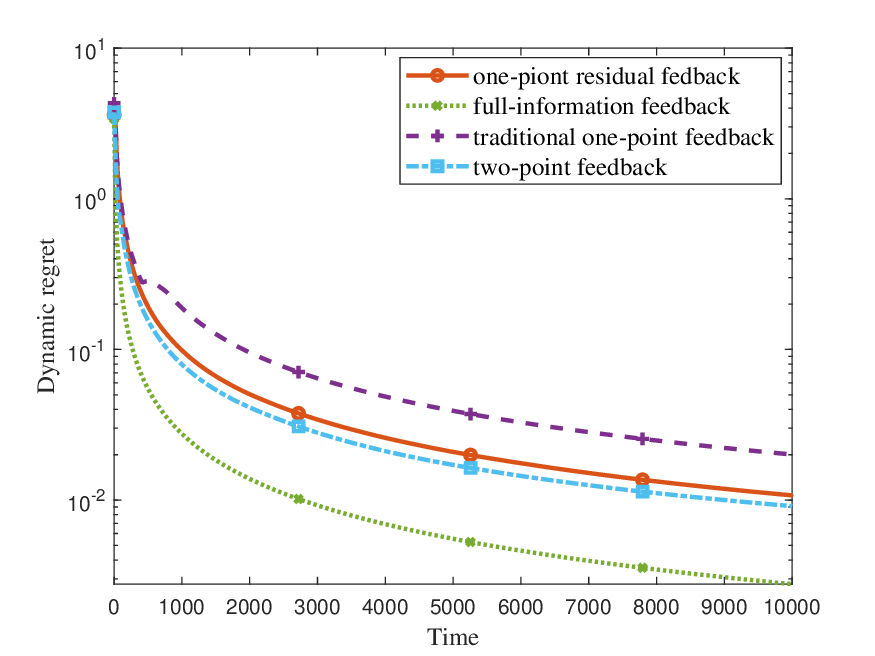}
	\captionsetup{font=footnotesize}
	\caption{The dynamic bounds of applying the proposed residual one-point feedback \eqref{eq11}, the two-point oracle \cite{33nesterov_random_2017}, and the traditional one-point oracle \cite{41flexman2005} to the convex DOBO problem.}
	\label{fig:huatu1}
\end{figure}
We evaluate the performance of different feedback models in solving the above policy optimization problem, specifically applying one-point feedback \cite{41flexman2005,48YUANDOO2022}, two-point feedback \cite{50shamir2017,2023yipengpang,52deming2021}, one-point residual feedback \eqref{eq11}, and full information feedback. For this simulation, we set the step size $\alpha_k = \frac{1}{500\sqrt{k +1}}$ and the smoothing parameter $\mu_k = \frac{1}{\sqrt{k+1}}$. The evolution of dynamic regret $\mathcal{D}\mathcal{R}_{j,T}^\text{cv}$ over time is depicted in Fig. 2, illustrating the convergence behavior of distributed online projected gradient descent algorithms under different feedback models.

As shown in Fig. 2, the one-point residual feedback \eqref{eq11} demonstrates significantly faster convergence compared to the traditional one-point oracle feedback. Moreover, the dynamic regret bound achieved with one-point residual feedback is comparable to that of the two-point feedback and full information feedback. This observation is consistent with our theoretical analysis presented in Section IV, further validating the advantages of the one-point residual feedback model in accelerating convergence while maintaining low query complexity.

\subsection{Nonconvex Case}
We construct a numerical example to evaluate the performance of the proposed OP-DOPGD algorithm for constrained DOBO with nonconvex losses. The system under consideration comprises ten agents, denoted as \(\mathcal{V} = \{1, \dots, 10\}\). The communication between agents is modeled using the time-varying graph depicted in Fig. 1. For each sensor \(i \in \mathcal{V}\) and time step \(k \in [T]\), the local, time-varying loss function is defined as the formulation in \cite{58comncvx2024}:
\[
f_{i,k}(x) = \frac{i}{63}x_1^3 + \frac{i-1}{15}(x_1^2 + x_2^2) - \frac{2(i-3)}{3}\tilde{r}_{i,k} \cos(x_2)\text{,}
\]
where \(\tilde{r}_{i,k} = \frac{1}{2} \arctan(k) + \frac{1}{2}\xi_{i,k}\), with \(\xi_{i,k} \sim \mathcal{N}(0, 1)\). For any \(k \in [T]\), it is evident that \(f_{i,k}\) is nonconvex with respect to \(x\). The agents collectively aim to solve the following optimization problem:
\[
\min_{x \in \Omega} f_k(x) = \sum_{i=1}^{10} f_{i,k}(x),
\]
where \(\Omega \subset \mathbb{R}^2 = [-3, 3] \times [-3, 3]\) represents a compact and convex set.
\begin{figure}
	\centering
	\includegraphics[width=1\linewidth]{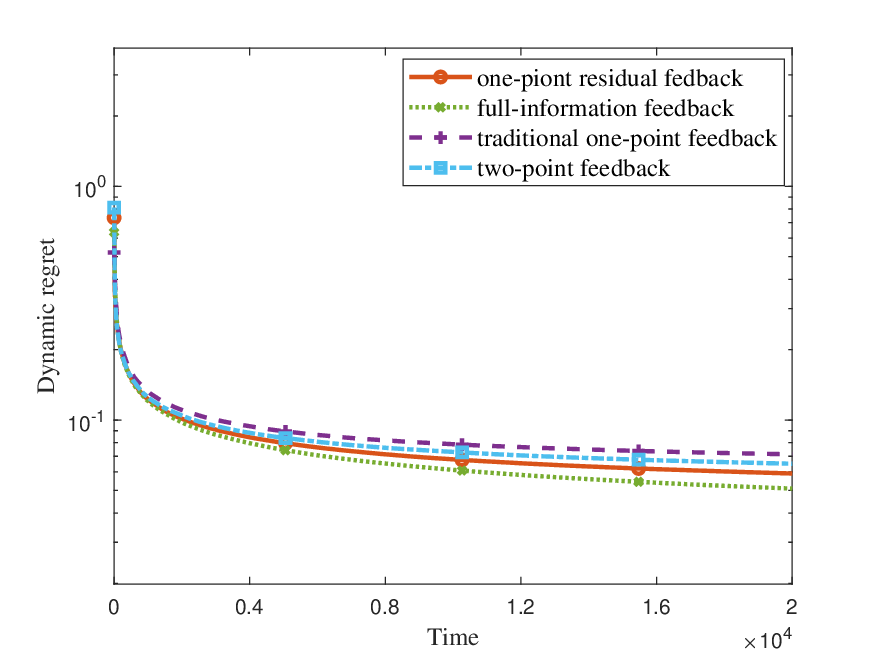}
\captionsetup{font=footnotesize}
\caption{The dynamic bounds of applying the proposed residual one-point feedback \eqref{eq11}, the two-point oracle \cite{33nesterov_random_2017}, and the traditional one-point oracle \cite{41flexman2005} to the nonconvex DOBO problem.}
	\label{fig:huatu2}
\end{figure}
To address this problem, we employ distributed online projected gradient descent methods under various feedback models, with the algorithm parameters set as \(\alpha = 5 \times 10^{-3}\) and \(\mu = 10^{-3}\). The initial points for all nodes are randomly selected within \([-3, 3] \times [-3, 3]\).

Fig. 3 illustrates the evolution of the dynamic regret \(\mathcal{D}\mathcal{R}_{j,T}^\text{ncv}\) over time, where the optimization process utilizing the proposed residual-feedback gradient performs comparably to that using the two-point gradient estimator and the exact gradient. Both estimators significantly outperform the traditional one-point gradient estimator, similar to the behavior observed in the convex case.
\section{Conclusion}
In this paper, we discuss the distributed online bandit optimization problem with nonconvex loss functions over a time-varying communication topology. We propose an online distributed optimization algorithm based on one-point residual feedback to solve this problem. We theoretically analyze the explicit dynamic regret bounds of the proposed method for both nonconvex and convex DOBOs, demonstrating that the algorithm significantly improves convergence speed while maintaining \(\mathcal{O}(1)\) sampling complexity compared to existing algorithms. In this paper, only simple ensemble constraints are considered. Exploring online optimization on dynamic constraint sets would be an interesting and challenging future direction.

\section*{Appendix}

\subsection{Proof of Lemma \ref{lem31}}
\begin{proof}
To conserve space, we abbreviate \( g_{i,k}(x_{i,k}) \) as \( g_{i,k} \) and omit the subscripts \( i \) from \( f_{i,k} \), \( x_{i,k} \), and \( u_{i,k} \). It is important to clarify that \( f_k \) and \( x_k \) in this subsection should not be confused with the global loss functions and decision vectors defined in Problem \eqref{eq1}.\par 
By considering the expression of \( g_{i,k}^s \) in \eqref{eq11} and applying the inequality \(\left(a + b + c \right)^2 \le 3\left( a^2 + b^2 + c^2 \right) \), we have
\begin{align}\label{apd1}
	&\mathbb{E}\left[ \left\|g_{i,k}^s \right\|^2\right]\nonumber\\
	&=\hspace{-2pt}\frac{1}{\mu_k^2}\mathbb{E}\hspace{-1pt}\left[\hspace{-2pt} \left( f_{k}(x_{k}+\mu_{k}u_{k})\hspace{-2pt}-\hspace{-2pt}f_{k-1}(x_{k-1}+\mu_{k-1}u_{k-1})\right)^2 \left\|u_k\right\|^2\hspace{-1pt}\right] \nonumber\\
	&\le\frac{3}{\mu_k^2}\mathbb{E}\hspace{-1pt}\left[\hspace{-1pt} \left( f_{k}(x_{k}+\mu_{k}u_{k})\hspace{-2pt}-\hspace{-2pt}f_{k}(x_{k}+\mu_{k-1}u_{k-1})\right)^2 \left\|u_k\right\|^2\hspace{-1pt}\right] \nonumber\\
	&+\frac{3}{\mu_k^2}\hspace{-1pt}\mathbb{E}\hspace{-1pt}\left[\hspace{-1pt} \left( f_{k}(x_{k}+\mu_{k\hspace{-1pt}-\hspace{-1pt}1}u_{k\hspace{-1pt}-\hspace{-1pt}1})\hspace{-2pt}-\hspace{-2pt}f_{k}(x_{k\hspace{-1pt}-\hspace{-1pt}1}+\mu_{k\hspace{-1pt}-\hspace{-1pt}1}u_{k\hspace{-1pt}-\hspace{-1pt}1})\right)^2 \left\|u_k\right\|^2\hspace{-1pt}\right] \nonumber\\
	&+\frac{3}{\mu_k^2}\mathbb{E}\Big[ ( f_{k}(x_{k-1}+\mu_{k-1}u_{k-1})\nonumber\\
	&\quad-f_{k-1}(x_{k-1}+\mu_{k-1}u_{k-1}))^2 \left\|u_k\right\|^2\Big].
\end{align}\par  
We first focus on the case where \( f_k \) is Lipschitz-continuous with constant \( L_0 \). Notice that \(\mathbb{E}\left[\| u_k \|^2\right] = d\) and 
\(\mathbb{E}\left[\| u_k - u_{k-1} \| \cdot \| u_k \|^2 \right] \le 2 \mathbb{E}\left[\| u_k \|^2 + \| u_{k-1} \|^2\right]\| u_k \|^2 \le 4 \left( d+4\right) ^2 \), as proven in \cite{33nesterov_random_2017}. Then, we have
\begin{align}\label{apd111}
	&\mathbb{E}\left[ \left( f_{k}(x_{k}+\mu_{k}u_{k}) - f_{k}(x_{k}+\mu_{k-1}u_{k-1}) \right)^2 \| u_{k} \|^2 \right]\nonumber\\
	&\le L_0^2\mathbb{E}\left[ \left( \mu_{k}u_{k}-\mu_{k-1}u_{k-1} \right)^2 \| u_{k} \|^2 \right]\nonumber\\
	&\le L_0^2\mu_{k-1}^2\mathbb{E}\left[ \left( u_{k}-u_{k-1} \right)^2 \| u_{k} \|^2 \right]\nonumber\\
	&\le 4\left(d+4 \right)^2L_0^2\mu_{k-1}^2\text{,} \nonumber\\
	&\mathbb{E}\left[ \left( f_{k}(x_{k-1}+\mu_{k-1}u_{k-1}) \hspace{-2pt}-\hspace{-2pt} f_{k-1}(x_{k}+\mu_{k-1}u_{k-1}) \right)^2 \hspace{-2pt}\| u_{k} \|^2\right]\nonumber\\
	&\le d\theta_{k}^2\text{,}
\end{align}
and
\begin{align}\label{apd112}
	&\mathbb{E}\left[ \left( f_{k}(x_{k}+\mu_{k-1}u_{k-1}) - f_{k}(x_{k-1}+\mu_{k-1}u_{k-1}) \right)^2 \| u_{k} \|^2 \right]\nonumber\\
	&\le L_0^2\mu_{k-1}^2\mathbb{E}\left[ \left\|x_{k}-x_{k-1} \right\|^2 \| u_{k-1} \|^2 \right]\nonumber\\
	&\le dL_0^2\mu_{k-1}^2\mathbb{E}\left[\left\|x_{k}-x_{k-1} \right\|^2\right]\text{,} 
\end{align}
where the second inequality holds by using the definition of $\theta_{i,k}$. Based on the preceding inequalities, it follows that
\begin{align}\label{apd4}
	\mathbb{E}\left[ \| g_{k}^{s} \|^2 \right] 
	& \le \frac{3dL_{0}^2}{\mu_{k}^2} \mathbb{E}\left[ \| x_{k} - x_{k-1} \|^2 \right]\nonumber\\
	&\quad + \frac{12(d+4)^2L_{0}^2\mu_{k-1}^2}{\mu_{k}^2} + \frac{3d\theta_{k}^2}{\mu_{k}^2} \text{.}
\end{align}
By applying Jensen’s inequality to \eqref{apd4} and restoring the subscript $i$, we obtain
\begin{align}\label{apd5}
	\mathbb{E}\left[ \| g_{i,k}^{s} \|\right] &\le \frac{\sqrt{3d}L_{0}}{\mu _{k}}\mathbb{E}\left[ \| x_{i,k} - x_{i,k-1} \| \right] \nonumber\\
	&\quad\hspace{-2pt}+\hspace{-2pt}\frac{2\sqrt{3}(d+4)L_{0} \mu _{k-1}}{\mu _{k}}  \hspace{-2pt}+\hspace{-2pt} \frac{\sqrt{3d}\theta _{i,k}}{\mu _{k}} .
\end{align}
We turn our attention to the term $ \mathbb{E}\left[ \| x_{i,k} - x_{i,k-1} \|\right]  $. To facilitate the analysis, we denote \(\tilde{x}_{i,k} = \sum_{j=1}^{n} [W_k]_{i,j} y_{j,k}\) and define the projection error as \(\|s_{i,k}\| = \|\mathcal{P}_{\Omega} \left[ \tilde{x}_{i,k} \right] -  \tilde{x}_{i,k}\|\). The evolution of $x_{i,k}$ allows us to derive
\begin{align}
	\left\| {{x}_{i,k}}-{{x}_{i,k-1}} \right\|& \le \sum\limits_{j=1}^{n} \left[W_{k-1}\right]_{i,j} \left\| x_{j,k-1} - x_{i,k-1} \right\|\nonumber\\
	&\quad+\left\| {{s}_{i,k-1}}-{{\alpha }_{k-1}}\sum\limits_{j=1}^{n}{\left[W_{k-1}\right]_{i,j}}g_{j,k-1}^{s}\right\|\text{,} \nonumber
\end{align}
where we used the double stochasticity of $W_{k-1}$.
For the second term on the right-hand side of the above equation, by the definition of $s_{i,k-1}$, the non-expansiveness of the Euclidean projection $\mathcal{P}_\Omega\left(\cdot\right)$ and the fact that $\sum\nolimits_{j=1}^{n} \left[W_{k-1}\right]_{i,j} x_{j,k}\in\Omega$, we obtain
\begin{align}
	& \left\| {{s}_{i,k-1}}-{{\alpha }_{k-1}}\sum\limits_{i=1}^{n}{[W_{k\hspace{-1pt}-\hspace{-1pt}1}]_{i,j}}g_{j,k-1}^{s} \right\| \nonumber\\ 
	& =\left\| \mathcal{P}_{\Omega }\left[ {{{\tilde{x}}}_{i,k-1}} \right]-\sum\limits_{i=1}^{n}{[W_{k\hspace{-1pt}-\hspace{-1pt}1}]_{i,j}}{{x}_{j,k-1}} \right\|\nonumber \\ 
	& \le\left\|\hspace{-1pt}  {{{\tilde{x}}}_{i,k-1}}\hspace{-3pt}-\hspace{-4pt}\sum\limits_{i=1}^{n}{[W_{k\hspace{-1pt}-\hspace{-1pt}1}]_{i,j}}{{x}_{j,k-1}}\hspace{-1pt}  \right\|\hspace{-4pt}\le\hspace{-2pt}{{\alpha }_{k-1}}\hspace{-2pt}\sum\limits_{j=1}^{n}\hspace{-2pt}{[W_{k\hspace{-1pt}-\hspace{-1pt}1}]_{i,j}}\hspace{-2pt}\left\| \hspace{-1pt}g_{j,k-1}^{s}\hspace{-1pt} \right\|\hspace{-2pt} .\nonumber
\end{align} Combining the two inequalities above gives
\begin{align}
	\left\| {{x}_{i,k}}-{{x}_{i,k-1}} \right\|& \le \sum\limits_{j=1}^{n} \left[W_{k-1}\right]_{i,j} \left\| x_{j,k-1} - x_{i,k-1} \right\|\nonumber\\
	&\quad+{\alpha }_{k-1}\sum\limits_{j=1}^{n}{\left[W_{k-1}\right]_{i,j}}\left\| g_{j,k-1}^{s}\right\|\text{.} \nonumber
\end{align}
Substituting it into \eqref{apd5} yields the desired result.

We further examine the case where $f_k$ is smooth with constant $L_1$. Adding
and subtracting $\left\langle \nabla {{f}_{k}}({{x}_{k}}), {{\mu }_{k}}{{u}_{k}} \right\rangle$, ${{f}_{k}}({{x}_{k}})$ and $\left\langle \nabla {{f}_{k}}({{x}_{k}}),{{\mu }_{k-1}}{{u}_{k-1}} \right\rangle$ inside the square term ${{\left( {{f}_{k}}({{x}_{k}}+{{\mu }_{k}}{{u}_{k}})-{{f}_{k}}({{x}_{k}}+{{\mu }_{k-1}}{{u}_{k-1}}) \right)}^{2}}$, we have
\begin{align}
	& {{\left( {{f}_{k}}({{x}_{k}}+{{\mu }_{k}}{{u}_{k}})-{{f}_{k}}({{x}_{k}}+{{\mu }_{k-1}}{{u}_{k-1}}) \right)}^{2}} \nonumber\\ 
	& \le 2{{\left\langle \nabla {{f}_{k}}({{x}_{k}}),{{\mu }_{k}}{{u}_{k}}-{{\mu }_{k-1}}{{u}_{k-1}} \right\rangle }^{2}}\nonumber\\
	&\quad +\hspace{-2pt}4{{\left( {{f}_{k}}({{x}_{k}}\hspace{-2pt}+\hspace{-2pt}{{\mu }_{k-1}}{{u}_{k-1}})\hspace{-2pt}-\hspace{-2pt}{{f}_{k}}({{x}_{k}})\hspace{-2pt}-\hspace{-2pt}\left\langle \nabla {{f}_{k}}({{x}_{k}}),{{\mu }_{k-1}}{{u}_{k-1}} \right\rangle  \right)}^{2}} \nonumber\\ 
	&\quad +4{{\left( {{f}_{k}}({{x}_{k}}+{{\mu }_{k}}{{u}_{k}})-{{f}_{k}}({{x}_{k}})-\left\langle \nabla {{f}_{k}}({{x}_{k}}),{{\mu }_{k}}{{u}_{k}} \right\rangle  \right)}^{2}}\nonumber\\
	& \le 2{{\left\langle \nabla {{f}_{k}}({{x}_{k}}),{{\mu }_{k}}{{u}_{k}}-{{\mu }_{k-1}}{{u}_{k-1}} \right\rangle }^{2}}\nonumber\\
	&\quad+L_{1}^{2}\mu _{k-1}^{4}{{\left\| {{u}_{k-1}} \right\|}^{4}}+L_{1}^{2}\mu _{k}^{4}{{\left\| {{u}_{k}} \right\|}^{4}}\text{,}\nonumber 
\end{align}
where the last inequality follows from the fact that the $L_1$-smooth function $f$ satisfies
$\left| f(y)-f(x)-\left\langle \nabla f(x),y-x \right\rangle  \right|\le \frac{{L}_{1}}{2}{{\left\| x-y \right\|}^{2}}$ for all $x,y\in\Omega$.
By once again using the properties of the random vector \( u_k \): \(\mathbb{E}\left[\| u_k \|^2\right] = d\) and \(\mathbb{E}\left[\| u_k - u_{k-1} \| \cdot \| u_k \|^2 \right] \le 4 \left( d+4 \right)^2\), we obtain
\begin{align}
	& \mathbb{E}\left[ {{\left( {{f}_{k}}({{x}_{k}}+{{\mu }_{k}}{{u}_{k}})-{{f}_{k}}({{x}_{k}}+{{\mu }_{k-1}}{{u}_{k-1}}) \right)}^{2}}{{\left\| {{u}_{k}} \right\|}^{2}} \right] \nonumber\\ 
	& \le 8\mu _{k-1}^{2}{{\left\| \nabla {{f}_{k}}({{x}_{k}}) \right\|}^{2}}{{\left( d+4 \right)}^{2}}+2L_{1}^{2}\mu _{k-1}^{4}{{\left( d+6\right)}^{3}}\text{.}\nonumber 
\end{align}
This, in combination with \eqref{eq13}, \eqref{apd1}, \eqref{apd111} and \eqref{apd112}, gives the desired result.
\end{proof}

\subsection{Proof of Lemma \ref{lem4}}
\begin{proof}
For the sake of simplicity, we will abbreviate \( g_{i,k}(x_{i,k}) \) as \( g_{i,k} \) without causing any confusion in this subsection. First, we consider the projection error \(\|s_{i,k}\| = \|\mathcal{P}_{\Omega} \left[ \tilde{x}_{i,k} \right] -  \tilde{x}_{i,k}\|\), which can be bounded as follows:
\begin{align}
	\left\| s_{i,k} \right\|&= \left\|\mathcal{P}_{\Omega} \left[ \tilde{x}_{i,k}  \right]-\sum_{j=1}^{n} [W_k]_{i,j} x_{j,k}+\alpha_k\sum_{j=1}^{n} [W_k]_{i,j} g_{j,k}^s \right\|\nonumber\\
	&\le \left\| \mathcal{P}_{\Omega} \left[ \tilde{x}_{i,k}  \right] \hspace{-2pt} -\hspace{-3pt} \sum_{j=1}^{n} \left[W_k\right]_{i,j} x_{j,k} \right\|\hspace{-2pt}+\hspace{-2pt}\alpha_{k}\hspace{-2pt} \sum_{j=1}^{n} \left[W_k\right]_{i,j} \left\| g_{j,k}^{s} \right\|\text{,}\nonumber
\end{align}
where the equality is derived from \eqref{eq12} in Algorithm 1, and the inequality follows from the fundamental properties of norms. Based on the non-expansiveness of 
 $\mathcal{P}_\Omega\left[\cdot\right]$ and the fact that $ \sum_{j=1}^{n} \left[W_k\right]_{i,j} x_{j,k}\in\Omega$, we have
  \begin{align}
 	\left\| \mathcal{P}_{\Omega} \left[ \tilde{x}_{i,k}  \right] \hspace{-2pt} -\hspace{-2pt} \sum_{j=1}^{n} \left[W_k\right]_{i,j} x_{j,k} \right\|&\le \left\|  \tilde{x}_{i,k} - \sum_{j=1}^{n} \left[W_k\right]_{i,j} x_{j,k} \right\|\nonumber\\
 	&\le \alpha_{k} \sum_{j=1}^{n} \left[W_k\right]_{i,j} \left\| g_{j,k}^{s} \right\|\text{,}\nonumber
 \end{align}
which allows us to further obtain
 \begin{align}\label{s_{i,k}}
 \left\| s_{i,k} \right\|	&\le 2\alpha_{k} \sum_{j=1}^{n} \left[W_k\right]_{i,j} \left\| g_{j,k}^{s} \right\|.
 \end{align}\par
We now derive the general evolution of ${x}_{i,k}$ by separately presenting the expressions of ${x}_{i,k}$ and $\bar{x}_{i,k}$. For $x_{i,k}$,
\begin{align}
	\hspace{-1pt}x_{i,k}\hspace{-2pt}=\hspace{-2pt} s_{i,k-1}\hspace{-2pt} +\hspace{-4pt} \sum\limits_{j=1}^{n} \hspace{-2pt}\left[W_{k-1}\right]_{i,j}\hspace{-1pt} x_{j,k-1}\hspace{-2pt}-\hspace{-4pt} \sum\limits_{j=1}^{n}\hspace{-2pt} \left[W_{k-1}\right]_{i,j}\hspace{-1pt} \alpha_{j,k-1}\hspace{-1pt} g_{j,k-1}^{s}.\nonumber
\end{align}
Applying the preceding equality recursively gives
\begin{align}\label{apd2}
	 {x}_{i,k}
 &\hspace{-2pt}=\hspace{-2pt}\sum\limits_{j=1}^{n}{\left[W\hspace{-1pt}(k\hspace{-2pt}-\hspace{-2pt}1,\hspace{-1pt}0)\right]_{i,j}{{x}_{j,0}}}\hspace{-2pt}-\hspace{-2pt}\sum\limits_{\tau =0}^{k-1}{\sum\limits_{j=1}^{n}\left[W\hspace{-1pt}(k\hspace{-2pt}-\hspace{-2pt}1,\hspace{-1pt}\tau)\right]_{i,j}{{\alpha }_{\tau }}g_{j,\tau }^{s}}\nonumber\\
 &\quad+\sum\limits_{\tau =0}^{k-2}{\sum\limits_{j=1}^{n}\left[W(k-1,\tau+1)\right]_{i,j}{{s}_{j,\tau }}}+{{s}_{i,k-1}}.
\end{align}
By taking average of both sides of \eqref{apd2} and applying the double stochasticity of the transition matrices, we obtain
\begin{align}\label{apd3}
	{{{\bar{x}}}_{k}}=\frac{1}{n}\sum\limits_{i=1}^{n}{{{x}_{i,0}}}-\frac{1}{n}\sum\limits_{\tau =0}^{k-1}{{\alpha }_{\tau }}{\sum\limits_{j=1}^{n}{g_{j,\tau }^{s}}}+\frac{1}{n}\sum\limits_{\tau =0}^{k-1}{\sum\limits_{j=1}^{n}{{{s}_{j,\tau }}}} \text{.} 
\end{align}
Combining the results in \eqref{apd2}, \eqref{apd3} and Lemma \ref{lem2} leads to
\begin{align}
	& \left\| {{x}_{i,k}}-{{{\bar{x}}}_{k}} \right\|\le \Gamma \sum\limits_{\tau =0}^{k-2}{{{\gamma }^{k-\tau -2}}\sum\limits_{j=1}^{n}{\left\| {{s}_{j,\tau }} \right\|}}+\left\| {{s}_{i,k-1}} \right\|\nonumber\\
	&+\Gamma \sum\limits_{\tau =0}^{k-1}{{{\gamma }^{k-\tau -1}}{{\alpha }_{\tau }}\sum\limits_{j=1}^{n}{\left\| g_{j,\tau }^{s} \right\|}} +\frac{1}{n}\sum\limits_{j=1}^{n}{\left\| {{s}_{j,k-1}} \right\|}\nonumber \text{.}
\end{align}
This, in conjunction with \eqref{s_{i,k}}, yields
\begin{align}
	& \left\| {{x}_{i,k}}-{{{\bar{x}}}_{k}} \right\|\le 3\Gamma \sum\limits_{\tau =0}^{k-2}{{{\gamma }^{k-\tau -1}}{{\alpha }_{\tau }}\sum\limits_{i=1}^{n}{\left\| g_{i,\tau }^{s} \right\|}} \nonumber\\ 
	& +\hspace{-2pt}2{{\alpha }_{k-1}}\hspace{-2pt}\sum\limits_{j=1}^{n}\hspace{-2pt}{W_{i,j}^{k-1}\left\| g_{j,k-1}^{s} \right\|}\hspace{-2pt}+\hspace{-2pt}\left( \frac{2}{n}\hspace{-2pt}+\hspace{-2pt}\Gamma  \right){{\alpha }_{k-1}}\hspace{-2pt}\sum\limits_{i=1}^{n}{\left\| g_{i,k-1}^{s} \right\|} \nonumber\text{,}
\end{align}
which further implies that 
\begin{align}
	& \left\| {{x}_{i,k}}-{{x}_{j,k}} \right\|\le 6\Gamma \sum\limits_{\tau =0}^{k-2}{{{\gamma }^{k-\tau -1}}{{\alpha }_{\tau }}\sum\limits_{i=1}^{n}{\left\| g_{i,\tau }^{s} \right\|}}\nonumber \\ 
	& +\hspace{-2pt}4{{\alpha }_{k-1}}\hspace{-2pt}\sum\limits_{j=1}^{n}\hspace{-2pt}{W_{i,j}^{k-1}\left\| g_{j,k-1}^{s} \right\|}\hspace{-2pt}+\hspace{-2pt}\left( \frac{4}{n}\hspace{-2pt}+\hspace{-2pt}2\Gamma  \right){{\alpha }_{k-1}}\hspace{-2pt}\sum\limits_{i=1}^{n}\hspace{-2pt}{\left\| g_{i,k-1}^{s} \right\|} \nonumber\text{.}
\end{align}
By summing over $i$ from $1$ to $n$, we obtain
\begin{align}\label{apd23}
	&\sum\limits_{i=1}^{n}{\left\| {{x}_{i,k}}-{{x}_{j,k}} \right\|}\nonumber\\
	&\le 6n\Gamma\hspace{-2pt} \sum\limits_{\tau =0}^{k-2}\hspace{-2pt}{{{\gamma }^{k-\tau -1}}{{\alpha }_{\tau }}\hspace{-2pt}\sum\limits_{i=1}^{n}{\left\| g_{i,\tau }^{s} \right\|}}\hspace{-2pt}+\hspace{-2pt}\left( 8\hspace{-2pt}+\hspace{-2pt}2n\Gamma  \right){{\alpha }_{k-1}}\hspace{-2pt}\sum\limits_{i=1}^{n}\hspace{-2pt}{\left\| g_{i,k-1}^{s} \right\|}\nonumber\\
	&\le (8+6n\Gamma )\sum\limits_{\tau =0}^{k-1}{{{\gamma }^{k-\tau -1}}{{\alpha }_{\tau }}\sum\limits_{i=1}^{n}{\left\| g_{i,\tau }^{s} \right\|}} \text{.} 
\end{align}\par 
We now establish a tight upper bound on the expected norm of $g_{i,k}^s$. Summing the inequalities \eqref{lem41} in Lemma \ref{lem31} from \( i= 1 \) to \( n \) and applying \eqref{apd23} yields
\begin{align}\label{Lemma51}
&\mathbb{E}\left[ \sum\limits_{i=1}^{n}{\left\| g_{i,k}^{s} \right\|} \right]\hspace{-2pt}\hspace{-2pt} \le\hspace{-2pt} \frac{\sqrt{3d}(8\hspace{-2pt}+\hspace{-2pt}6n\Gamma ){{L}_{0}}}{{{\mu }_{k}}}\sum\limits_{\tau =0}^{k-1}{{{\gamma }^{k-\tau -2}}{{\alpha }_{\tau }}\mathbb{E}\hspace{-2pt}\left[\hspace{-2pt}\sum\limits_{i=1}^{n}\hspace{-2pt}\left\| g_{i,\tau }^{s}\right\|\hspace{-2pt}\right] }\nonumber \\ 
&+\frac{2\sqrt{3}\left( d+4 \right)n{{L}_{0}}{{\mu }_{k-1}}}{{{\mu }_{k}}}+\frac{\sqrt{3d}}{{{\mu }_{k}}}\sum\limits_{i=1}^{n}{{{\theta }_{i,k}}}\text{.}
\end{align}
By multiplying both sides by \(\alpha_k\) and summing over \(k\) from $k=1$ to $T$, we obtain
\begin{align}\label{apd8}
	& \sum\limits_{k=1}^{T}{{{\alpha }_{k}}\mathbb{E}\left[ \sum\limits_{i=1}^{n}{\left\| g_{i,k}^{s}\right\|} \right]} \nonumber\\ 
	& \le\hspace{-2pt} \sqrt{3d}(8\hspace{-2pt}+\hspace{-2pt}6n\Gamma ){{L}_{0}}\sum\limits_{k=1}^{T}\left\{\frac{{{\alpha }_{k}}}{{{\mu }_{k}}}\sum\limits_{\tau =0}^{k-1}{{{\gamma }^{k-\tau -2}}{{\alpha }_{\tau }}\mathbb{E}\left[ \sum\limits_{i=1}^{n}{\left\| g_{i,\tau }^{s} \right\|} \right]}\right\} \nonumber\\ 
	& \quad+2\sqrt{3}\left( d+4 \right)n{{L}_{0}}\sum\limits_{k=1}^{T}{\frac{{{\alpha }_{k}}{{\mu }_{k-1}}}{{{\mu }_{k}}}}+\sqrt{3d}\sum\limits_{k=1}^{T}{\frac{{{\alpha }_{k}}}{{{\mu }_{k}}}\sum\limits_{i=1}^{n}{{{\theta }_{i,k}}}}\text{.}  
\end{align}\par
We now turn our attention to the term 
$\sum\nolimits_{k=1}^{T} \left\{\frac{{{\alpha }_{k}}}{{{\mu }_{k}}} \sum\nolimits_{\tau =0}^{k-1} {{\gamma }^{k-\tau -2} \alpha _{\tau} \mathbb{E} \left[ \sum\nolimits_{i=1}^{n} \left\| g_{i,\tau }^{s} \right\| \right]}\right\}$. By using the fact that $\gamma<1$ and interchanging the order of summation, we find that
\begin{align}
	& \sum\limits_{k=1}^{T}\left\{{\frac{{{\alpha }_{k}}}{{{\mu }_{k}}}\sum\limits_{\tau =0}^{k-1}{{{\gamma }^{k-\tau -2}}{{\alpha }_{\tau }}\mathbb{E}\left[ \sum\limits_{i=1}^{n}{\left\| g_{i,\tau }^{s} \right\|} \right]}}\right\} \nonumber\\ 
	& =\frac{1}{{{\gamma }^{2}}}\sum\limits_{k=1}^{T}\left\{{\frac{{{\alpha }_{k}}{{\gamma }^{k}}}{{{\mu }_{k}}}\sum\limits_{\tau =0}^{k-1}{{{\gamma }^{-\tau }}{{\alpha }_{\tau }}\mathbb{E}\left[ \sum\limits_{i=1}^{n}{\left\| g_{i,\tau }^{s} \right\|} \right]}}\right\}  \nonumber\\ 
	& =\frac{1}{{{\gamma }^{2}}}\sum\limits_{k=0}^{T-1}\left\{{\left( \sum\limits_{\tau =k+1}^{T}{\frac{{{\alpha }_{\tau }}{{\gamma }^{\tau -k}}}{{{\mu }_{\tau }}}} \right){{\alpha }_{k}}\mathbb{E}\left[ \sum\limits_{i=1}^{m}{\left\| g_{i,k}^{s} \right\|} \right]}\right\} \nonumber
\end{align}
To simplify the notations, we denote ${{A}_{k}}=\sum\nolimits_{\tau =k+1}^{T}{\frac{{{\alpha }_{\tau }}{{\gamma }^{\tau -k}}}{{{\mu }_{\tau }}}}$ and $\tau-k=s$, we have ${{A}_{k}}=\sum\nolimits_{s=1}^{T-k}{\frac{{{\alpha }_{s+k}}{{\gamma }^{s}}}{{{\mu }_{s+k}}}}$, $\gamma {{A}_{k}}=\sum\nolimits_{s=1}^{T-k}{\frac{{{\alpha }_{s+k}}{{\gamma }^{s+1}}}{{{\mu }_{s+k}}}}=\sum\nolimits_{s=2}^{T-k+1}{\frac{{{\alpha }_{s+k-1}}{{\gamma }^{s}}}{{{\mu }_{s+k-1}}}}$. And hence, we get 
\begin{align}
 (1-\gamma ){{A}_{k}}	&=\sum\limits_{s=1}^{T-k}{\frac{{{\alpha }_{s+k}}{{\gamma }^{s}}}{{{\mu }_{s+k}}}}-\sum\limits_{s=2}^{T-k+1}{\frac{{{\alpha }_{s+k-1}}{{\gamma }^{s}}}{{{\mu }_{s+k-1}}}} \nonumber\\ 
	& =\hspace{-2pt}\frac{{{\alpha }_{k+1}}\gamma }{{{\mu }_{k+1}}}\hspace{-2pt}+\hspace{-2pt}\sum\limits_{s=2}^{T-k}\hspace{-2pt}{{{\gamma }^{s}}\hspace{-2pt}\left(\hspace{-2pt} \frac{{{\alpha }_{s+k}}}{{{\mu }_{s+k}}}\hspace{-2pt}-\hspace{-2pt}\frac{{{\alpha }_{s+k-1}}}{{{\mu }_{s+k-1}}} \hspace{-2pt}\right)}\hspace{-2pt}-\hspace{-2pt}\frac{{{\alpha }_{T}}{{\gamma }^{T\hspace{-1pt}-\hspace{-1pt}k+1}}}{{{\mu }_{T}}} \nonumber\\ 
	& \le \frac{{{\alpha }_{k+1}}\gamma }{{{\mu }_{k+1}}}\text{,}\nonumber
\end{align}
where the inequality holds due to the fact that when the sequence \(\{\frac{{\alpha }_{k}}{{\mu }_{k}}\}_{k=1}^T\) is non-increasing, \(\frac{{\alpha }_{s+k}}{{\mu }_{s+k}}-\frac{{\alpha }_{s+k-1}}{{\mu }_{s+k-1}} \le 0\) always holds for any \(k, s \in [T]\). This allows us to directly derive that ${{A}_{k}}\le \frac{{{\alpha }_{k+1}}\gamma }{\left( 1-\gamma  \right){{\mu }_{k+1}}}$. As a result,
\begin{align}
	& \sum\limits_{k=1}^{T}\left\{{\frac{{{\alpha }_{k}}}{{{\mu }_{k}}}\sum\limits_{\tau =0}^{k-1}{{{\gamma }^{k-\tau -2}}{{\alpha }_{\tau }}\mathbb{E}\left[ \sum\limits_{i=1}^{n}{\left\| g_{i,\tau }^{s} \right\|} \right]}} \right\}\nonumber\\ 
	& \le \frac{1}{\gamma \left( 1-\gamma\right)  }\sum\limits_{k=0}^{T-1}{\frac{{{\alpha }_{k+1}}}{{{\mu }_{k+1}}}{{\alpha }_{k}}\mathbb{E}\left[ \sum\limits_{i=1}^{n}{\left\| g_{i,k}^{s} \right\|} \right]}\nonumber
\end{align}
By substituting it into \eqref{apd8} and setting $g_{i,0}^s=0$ for any $i\in \mathcal{V}$, we have
\begin{align}\label{lemma52}
	& \sum\limits_{k=1}^{T}{{{\alpha }_{k}}\mathbb{E}\left[ \sum\limits_{i=1}^{n}{\left\| g_{i,k}^{s}\right\|} \right]} \nonumber\\ 
	& \le \frac{ \sqrt{3d}(8\hspace{-2pt}+\hspace{-2pt}6n\Gamma ){{L}_{0}}}{\gamma\left( 1-\gamma\right)  }\sum\limits_{k=1}^{T-1}{\frac{{{\alpha }_{k+1}}}{{{\mu }_{k+1}}}{{\alpha }_{k}}\mathbb{E}\left[ \sum\limits_{i=1}^{n}{\left\| g_{i,k}^{s} \right\|} \right]}\nonumber\\
	&\quad+2\sqrt{3}\left( d+4 \right)n{{L}_{0}}\sum\limits_{k=1}^{T}{\frac{{{\alpha }_{k}}{{\mu }_{k-1}}}{{{\mu }_{k}}}}+\sqrt{3d}\sum\limits_{k=1}^{T}{\frac{{{\alpha }_{k}}}{{{\mu }_{k}}}\sum\limits_{i=1}^{n}{{{\theta }_{i,k}}}}\text{.}
\end{align}\par 
Define $V_k = {{{\alpha }_{k}}\mathbb{E}\left[ \sum\nolimits_{i=1}^{n}{\left\| g_{i,k}^{s}\right\|} \right]}$, the constant factor $M =\frac{\sqrt{3d}(8+6n\Gamma){{L}_{0}}}{\gamma\left( 1-\gamma\right) }$ and the perturbation term $C_T = 2\sqrt{3}(d+4)n{{L}_{0}}\sum\nolimits_{k=1}^{T}{\frac{{{\alpha }_{k}}{{\mu }_{k-1}}}{{{\mu }_{k}}}} + \sqrt{3d}\sum\nolimits_{k=1}^{T}{\frac{{{\alpha }_{k}}}{{{\mu }_{k}}}\sum\nolimits_{i=1}^{n}{{{\theta }_{i,k}}}}$. Then,
\begin{equation}
\sum\limits_{k=1}^{T}{{{V}_{k}}}\le M\sum\limits_{k=1}^{T-1}{\frac{{{\alpha }_{k+1}}}{{{\mu }_{k+1}}}{{V}_{k}}}+{{C}_{T}}\text{.}\nonumber
\end{equation}
It further suggests that \[{{V}_{T}}\le \sum\limits_{k=1}^{T-1}{\left( M\frac{{{\alpha }_{k+1}}}{{{\mu }_{k+1}}}-1 \right){{V}_{k}}}+{{C}_{T}}.\]
Next, we derive the general evolution of ${V}_{T}$:
\begin{align}\label{apd9}
{{V}_{T}}	& \le \sum\limits_{k=1}^{T-2}{\left( M\frac{{{\alpha }_{T}}}{{{\mu }_{T}}}\left( M\frac{{{\alpha }_{k+1}}}{{{\mu }_{k+1}}}-1 \right) \right){{V}_{k}}}\nonumber\\
	&+\left( M\frac{{{\alpha }_{T}}}{{{\mu }_{T}}}-1 \right){{C}_{T-1}}+{{C}_{T}} \nonumber\\ 
	& \le \cdots \nonumber \\  
	& \le \sum\limits_{k=1}^{T-l}\bigg\{\left( M\frac{{{\alpha }_{T-l+2}}}{{{\mu }_{T-l+2}}} \right)\left( M\frac{{{\alpha }_{T-l+3}}}{{{\mu }_{T-l+3}}} \right)\cdots\nonumber\\
	&\quad\left( M\frac{{{\alpha }_{T-1}}}{{{\mu }_{T-1}}} \right)\left( M\frac{{{\alpha }_{T}}}{{{\mu }_{T}}} \right)\left( M\frac{{{\alpha }_{k+1}}}{{{\mu }_{k+1}}}-1 \right){{V}_{k}}\bigg\}\nonumber \\  
	& +\hspace{-1pt}\left( \hspace{-1pt}M\frac{{{\alpha }_{T-l+3}}}{{{\mu }_{T-l+3}}} \hspace{-1pt}\right)\hspace{-1pt}\cdots\hspace{-1pt}\left( \hspace{-1pt}M\frac{{{\alpha }_{T}}}{{{\mu }_{T}}}\hspace{-1pt} \right)\left(\hspace{-1pt} M\frac{{{\alpha }_{T-l+2}}}{{{\mu }_{T-l+2}}}\hspace{-1pt}-\hspace{-1pt}1 \hspace{-1pt}\right)\hspace{-1pt}{{C}_{T-l+1}}\nonumber \\ 
	& +\cdots+M\frac{{{\alpha }_{T}}}{{{\mu }_{T}}}\left( M\frac{{{\alpha }_{T-1}}}{{{\mu }_{T-1}}}-1 \right){{C}_{T-2}}\nonumber\\
	&+\left( M\frac{{{\alpha }_{T}}}{{{\mu }_{T}}}-1 \right){{C}_{T-1}}+{{C}_{T}} \nonumber \\ 
	&\le\cdots\nonumber\\
	& \le \underbrace{\left( M\frac{{{\alpha }_{2}}}{{{\mu }_{2}}}-1 \right)\left( M\frac{{{\alpha }_{3}}}{{{\mu }_{3}}} \right)\left( M\frac{{{\alpha }_{4}}}{{{\mu }_{4}}} \right)\cdots\left( M\frac{{{\alpha }_{T}}}{{{\mu }_{T}}} \right){{V}_{1}}}_{ \triangleq {{T}_{1}}} \nonumber \\  
	& +\left( M\frac{{{\alpha }_{3}}}{{{\mu }_{3}}}-1 \right)\left( M\frac{{{\alpha }_{4}}}{{{\mu }_{4}}} \right)\cdots\left( M\frac{{{\alpha }_{T-1}}}{{{\mu }_{T-1}}} \right)\left( M\frac{{{\alpha }_{T}}}{{{\mu }_{T}}} \right){{C}_{2}} \nonumber \\  
	& +\cdots+M\frac{{{\alpha }_{T}}}{{{\mu }_{T}}}\left( M\frac{{{\alpha }_{T-1}}}{{{\mu }_{T-1}}}-1 \right){{C}_{T-2}}\nonumber\\
	&+\left( M\frac{{{\alpha }_{T}}}{{{\mu }_{T}}}-1 \right){{C}_{T-1}}+{{C}_{T}} \triangleq {{T}_{1}}+{{T}_2}  
\end{align}
Let \(\alpha_k = \frac{1}{2Mk^a}\) and \(\mu_k = \frac{1}{k^b}\). Consequently, it follows that \( \frac{M\alpha_{2}}{\mu_{2}}=\frac{1}{2} \), which directly implies \( T_1 < 0 \). Our focus now shifts to developing an upper bound for the expected norm of \( g_{i,k}^s \) by examining the term \(T_2 \). By applying the relation \( C_{k+1} = C_{k} + 2\sqrt{3}(d+4)nL_{0} \frac{\alpha_{k+1} \mu_{k}}{\mu_{k}} + \sqrt{3d} \frac{\alpha_{k+1}}{\mu_{k+1}} \sum_{i=1}^{n} \theta_{i,k+1} \) to \eqref{apd9}, it becomes evident that certain terms will cancel each other out. This leads to 
\begin{align}\label{g_{i,k}}
	&\mathbb{E}\left[\sum_{i=1}^{n}\left\|g_{i,T}^s\right\|\right] 
	\le2A\underbrace{\sum\limits_{k=1}^{T}{\left\{ {{M}^{T-k}}\frac{{\mu }_{k-1}}{\alpha_T}\frac{{{\alpha }_{k}}}{{{\mu }_{k}}}\cdots\frac{{{\alpha }_{T-1}}}{{{\mu }_{T-1}}}\frac{{{\alpha }_{T}}}{{{\mu }_{T}}} \right\}}}_{Z_1} \nonumber\\
	&+2B\underbrace{\sum\limits_{k=1}^{T}{\left\{ {{M}^{T-k}}\frac{1}{\alpha_T}\sum\limits_{i=1}^{n}{{{\theta }_{i,k}}}\frac{{{\alpha }_{k}}}{{{\mu }_{k}}}\frac{{{\alpha }_{k+1}}}{{{\mu }_{k+1}}}\frac{{{\alpha }_{k+2}}}{{{\mu }_{k+2}}}\cdots\frac{{{\alpha }_{T-1}}}{{{\mu }_{T-1}}}\frac{{{\alpha }_{T}}}{{{\mu }_{T}}} \right\}}}_{Z_2}\text{,}
\end{align}
where $A=2\sqrt{3}(d+4)$ and $B=\sqrt{3d}$. We will further establish a tight upper bound for \( \mathbb{E}\left[\sum_{i=1}^{n}\left\|g_{i,T}^s\right\|\right] \) by examining terms $Z_1$ and $Z_2$ separately.\par 
Substituting the explicit expressions of \(\alpha_k = \frac{1}{2Mk^a}\) and \(\mu_k = \frac{1}{k^b}\) into the term $Z_1$, we find that
\begin{align}
	Z_1=\frac{{T}^{a}}{\left( {T!}\right)^{a-b} }\sum\limits_{k=1}^{T}{ \frac{{\left( {\left( k-1 \right)!} \right)}^{a-b}}{{{2}^{T-k}}{{\left( k-1 \right)}^{b}}} }.\nonumber
\end{align}
For ease of analysis, we denote $P_k=\frac{{\left( {\left( k-1 \right)!} \right)}^{a-b}}{{{2}^{T-k}}{{\left( k-1 \right)}^{b}}}$. It is easy to deduce that 
\begin{equation}\label{P_k}
	\frac{{{P}_{k}}}{{{P}_{k+1}}}=\frac{{{k}^{2b-a}}}{2{{\left( k-1 \right)}^{b}}}.
\end{equation}
Following similar lines as that of  term $Z_1$, we immediately have
\begin{align}\label{Z_2}
	{Z_2}&=\frac{{{T}^{b}}}{{{\left( (T-1)! \right)}^{a-b}}}\sum\limits_{k=1}^{T}{\left\{2^{k-T} {{\left( \left( k-1 \right)! \right)}^{a-b}}\sum\limits_{i=1}^{n}{{{\theta }_{i,k}}} \right\}} \nonumber\\ 
	& \le \frac{{{T}^{b}}}{{{\left( (T-1)! \right)}^{a-b}}}\sum\limits_{i=1}^{n}{\tilde{\Theta}_{i,T}}\sum\limits_{k=1}^{T}2^{k-T}{{{\left( \left( k-1 \right)! \right)}^{a-b}}} \nonumber\\
	& \le \frac{{{T}^{b}}}{{{\left( (T-1)! \right)}^{a-b}}}\sum\limits_{i=1}^{n}{\tilde{\Theta}_{i,T}}\sum\limits_{k=1}^{T}{Q_k} \text{,}
\end{align}
where $\tilde{\Theta}_{i,T}={\max }_{k\in[T]}$, $Q_k=2^{k-T}{{{\left( \left( k-1 \right)! \right)}^{a-b}}}$. It is obvious that 
\begin{equation}\label{Q_k}
	\frac{Q_k}{Q_{k+1}}=\frac{1}{2k^{a-b}}.
\end{equation}
Due to \( a \le b \), \eqref{P_k} simplifies to \(\frac{P_k}{P_{k+1}} \le \frac{k^a}{2(k-1)^a}\). Observing that \(\lim_{k \to \infty}(1 + \frac{1}{k-1})^a = 1\) for any \(a \in (0,1)\), we find that $(1+\frac{1}{k-1})^a<\frac{4}{3}$ for any $k>(\frac{4}{3})^{\frac{1}{a}}-1$. Specifically, when $ a=b=\frac{1}{2}$, $(1+\frac{1}{k-1})^a<\frac{4}{3}$ always
holds for any $k\in [T]$, which further implies $\frac{{{P}_{k}}}{{{P}_{k+1}}}\le\frac{2}{3}$ for any $k\in [T]$. This, combined with $\sum_{k=1}^{T-1}\left(\frac{2}{3}\right) ^k<3$, yields
\begin{align}
	&\sum\limits_{k=1}^{T}{{{P}_{k}}}=P_2+\cdots+P_{T-1}+P_T\nonumber\\
	&<\left( \frac{2}{3}\right) ^{T-1}P_T+\cdots+\left( \frac{2}{3}\right) ^2P_T+\frac{2}{3}P_T+P_T<3P_T\nonumber
\end{align}
As a result,
\begin{align}\label{z_1}
	Z_1=\frac{{T}^{a}}{\left( {T!}\right)^{a-b} }\sum\limits_{k=1}^{T}{P_k}\le \frac{3T^b}{{\left( T-1 \right)}^{b}}<4
\end{align}
for any $T>1$. By substituting $a\le b$ into \eqref{Q_k}, we arrive at $\frac{Q_k}{Q_{k+1}}\le\frac{1}{2}$. Following an argument similar to $Z_1$, we get that 
\begin{align}
	&\sum\limits_{k=1}^{T}{Q_k}=Q_2+\cdots+Q_{T-1}+Q_T\nonumber\\
	&\le\frac{1}{2^{T-1}}P_T+\cdots+\frac{1}{2 ^2}P_T+\frac{1}{2}P_T+P_T<2P_T\nonumber
\end{align}
This, combined with \eqref{Z_2}, yields 
\begin{align}\label{z_2}
	Z_2 &\le \frac{{2{T}^{b}Q_T}}{{{\left( (T-1)! \right)}^{a-b}}}\sum\limits_{i=1}^{n}{\tilde{\Theta}_{i,T}}\nonumber\\
	&\le{{2{T}^{b}}}\sum\limits_{i=1}^{n}{\tilde{\Theta}_{i,T}}.
\end{align}
We then combined the results in \eqref{g_{i,k}}, \eqref{z_1} and \eqref{z_2} to get
\begin{align}\label{g_{i,T}}
	&\sum_{i=1}^{n}\mathbb{E}\left[\left\|g_{i,T}^s\right\|\right] 
	\le4\sqrt{3}\left( d+4\right)+4\sqrt{3d}{T}^{b}\sum\limits_{i=1}^{n}{\tilde{\Theta}_{i,T}}\text{,}
\end{align}

This, together with \eqref{apd23}, further gives
\begin{align}\label{38}
	&	\sum\limits_{k=1}^{T}\sum \limits_{i=1}^{n} \mathbb{E}\left[\left\| {{x}_{i,k}}-{{x}_{j,k}} \right\|\right] \nonumber\\
	&\le (8+6n\Gamma )\sum\limits_{k=1}^{T} \sum\limits_{\tau =0}^{k-1}{{\gamma }^{k-\tau -1}}{{\alpha }_{\tau }}\sum\limits_{i=1}^{n}\mathbb{E}\left[{\left\| g_{i,\tau }^{s} \right\|}\right]  \nonumber\\
	&\le \frac{(8+6n\Gamma )}{1-\gamma}\sum\limits_{k=0}^{T-1} {{\alpha }_{k }}\sum\limits_{i=1}^{n}\mathbb{E}\left[{\left\| g_{i,k }^{s} \right\|} \right] \nonumber\\
	&\le \frac{(32+24n\Gamma )\sqrt{3}\left( d+4\right) }{1-\gamma}\sum\limits_{k=0}^{T-1} {{\alpha }_{k }}\nonumber\\
	&\quad+\frac{(32+24n\Gamma )\sqrt{3d}}{1-\gamma}\sum\limits_{k=0}^{T-1} \frac{{\alpha }_{k }}{\mu_k}\sum\limits_{i=1}^{n}{\tilde{\Theta}_{i,k}} .
\end{align}
Substituting the explicit expressions for \(\alpha_k = \frac{1}{2Mk^a}\), \(\mu_k = \frac{1}{k^b}\) with $a\le b$, and ${M} =\frac{\sqrt{3d}(8+6n\Gamma){{L}_{0}}}{\gamma\left( 1-\gamma\right) }$ into \eqref{38}, we have 
\begin{align}
\sum\limits_{k=1}^{T}\hspace{-2pt}\mathbb{E}\hspace{-2pt}\left[\hspace{-1pt}\sum \limits_{i=1}^{n}\hspace{-2pt}\left\|\hspace{-1pt} {{x}_{i,k}}\hspace{-2pt}-\hspace{-2pt}{{x}_{j,k}}\hspace{-1pt} \right\|\hspace{-1pt}\right]\hspace{-2pt}&\hspace{-2pt}\le\hspace{-3pt} \frac{2\gamma\left(\hspace{-1pt} d\hspace{-2pt}+\hspace{-2pt}4\hspace{-1pt}\right) }{\sqrt{d}{{L}_{0}}}\hspace{-3pt}\sum\limits_{k=0}^{T-1}\hspace{-2pt}{k^{-\hspace{-1pt}a}}\hspace{-2pt}+\hspace{-2pt}\frac{2\gamma }{{{L}_{0}}}\hspace{-2pt}\sum\limits_{k=0}^{T-1}\hspace{-1pt}k^{b\hspace{-1pt}-\hspace{-1pt}a\hspace{-2pt}}\hspace{-2pt} \sum\limits_{i=1}^{n}\hspace{-3pt}{\tilde{\Theta}_{i,k}}\nonumber\\
	&\le \frac{2\gamma\left( d+4\right) }{\sqrt{d}{{L}_{0}}}T^{1-a}+\frac{2\gamma }{{{L}_{0}}}T^{b-a}{{\Theta}_{T}}\text{,}\nonumber
\end{align}
because $k^{b-a}$ is positive and non-increasing in the last inequality.
\end{proof}
\subsection{Proof of Theorem 1}
\begin{proof}
To facilitate the analysis, we denote \(\tilde{x}_{i,k} = \sum_{j=1}^{n} [W_k]_{i,j} y_{j,k}\). Define the positive scalar function $\Lambda_k$ as  
$\Lambda_k = \sum_{i=1}^{n} \left\| x_{i,k} - x_k^* \right\|^2$, 
where $x_k^* = \arg \min_{x \in \Omega} f_k(x)$. From the evolution of $x_{i,k}$, the function $\Lambda_{k+1}$ can then be bounded as follows:
	\begin{align}\label{th2}
		\Lambda_{k+1}&=\sum\limits_{i=1}^{n}\left\| \mathcal{P}_{\Omega }\left[ \tilde{x}_{i,k}\right]-{{x}_{k+1}^*} \right\|^{2}\le \sum\limits_{i=1}^{n}{{{\left\| \tilde{x}_{i,k}-{{x}_{k+1}^*}\right\|}^{2}}} \nonumber\\  
		& \le \sum\limits_{i=1}^{n}{{{\left\| {{x}_{i,k}}-{{\alpha }_{k}}g_{i,k}^{s}-{{x}_{k+1}^*} \right\|}^{2}}}  \nonumber\\
		&= \sum\limits_{i=1}^{n}{{{\left\| {{x}_{i,k}}-{{x}_{k+1}^*} \right\|}^{2}}}+{{\alpha }_{k}^2}\sum\limits_{i=1}^{n}{{{\left\|g_{i,k}^{s} \right\|}^{2}}}\nonumber\\
		&\quad-2{{\alpha }_{k}}\sum\limits_{i=1}^{n}\left\langle{{x}_{i,k}}-{{x}_{k+1}^*} , g_{i,k}^{s}\right\rangle \text{,}
	\end{align}	 	
	where the first inequality results from the nonexpansiveness of the Euclidean projection $\mathcal{P}_k\left[\cdot \right] $, and the second inequality is due to the double stochasticity of \(W_{k+1}\) and the convexity of the Euclidean norm. Adding and subtracting $x_k^*$ inside the square term ${\| {{x}_{i,k}}-{{x}_{k+1}^*} \|}^{2}$, we can further obtain
	\begin{align}
		\Lambda_{k+1}	&\hspace{-2pt}\le\hspace{-2pt} \Lambda_{k}\hspace{-2pt}+\hspace{-2pt}\sum\limits_{i=1}^{n}{{{\left\| {{x}_{i,k}^*}\hspace{-2pt}-\hspace{-2pt}{{x}_{k+1}^*} \right\|}^{2}}}\hspace{-2pt}+\hspace{-2pt}2\sum\limits_{i=1}^{n}\left\langle {x}_{i,k}\hspace{-2pt}-\hspace{-2pt}{x}_{k}^*,{x}_{ k}^*\hspace{-2pt}-\hspace{-2pt}{{x}_{k+1}^*}\right\rangle  \nonumber\\
		&\quad+{{\alpha }_{k}^2}\sum\limits_{i=1}^{n}{{{\left\|g_{i,k}^{s} \right\|}^{2}}}-2{{\alpha }_{k}}\sum\limits_{i=1}^{n}\left\langle{{x}_{i,k}}-{{x}_{k}^*} , g_{i,k}^{s}\right\rangle\nonumber\\
		&\quad -2{{\alpha }_{k}}\sum\limits_{i=1}^{n}\left\langle{{x}_{k}^*}-{{x}_{k+1}^*} , g_{i,k}^{s}\right\rangle  \text{.}
	\end{align}
By rearranging the terms and eliminating the negative components, we obtain
	\begin{align}
		\sum\limits_{i=1}^{n}\left\langle {{x}_{i,k}}-{{x}_{k}^*},g_{i,k}^{s} \right\rangle &\hspace{-1pt} \le\hspace{-1pt} \frac{\Lambda_{k}-\Lambda_{k+1}}{2{{\alpha }_{k}}}\hspace{-2pt}+\hspace{-2pt} \frac{1}{2{{\alpha }_{k}}}\sum_{i=1}^n\left\|x_k^*-x_{k+1}^*\right\|^2\nonumber\\
		& \quad+ \frac{D}{{{\alpha }_{k}}}\left\|x_{k}^*-x_{k+1}^*\right\|\hspace{-2pt} +\hspace{-2pt}\frac{{\alpha }_{k}}{2}\sum\limits_{i=1}^{n}{{{\left\| g_{i,k}^{s} \right\|}^{2}}}\text{,} \nonumber
	\end{align}
    where we have used $\left\|x-y\right\|\le D$ for any $x,y\in\Omega$. Summing the preceding inequalities over \( k = 1, \ldots, T \) and taking the expectation on both sides, we find that
	\begin{align}\label{th1}
		& \sum\limits_{k=1}^{T}{\sum\limits_{i=1}^{n}{\mathbb{E}\left[ \left\langle {{x}_{i,k}}-{{x}_{k}^*},\nabla f_{i,k}^{s}\left( {{x}_{i,k}} \right) \right\rangle  \right]}} \nonumber\\ 
		& \le \sum\limits_{k=1}^{T}\frac{\mathbb{E}[\Lambda_{k}]-\mathbb{E}[\Lambda_{k+1}]}{2{{\alpha }_{k}}}+\sum\limits_{k=1}^{T}\frac{1}{2{{\alpha }_{k}}}\sum_{i=1}^n\left\|x_k^*-x_{k+1}^*\right\|^2\nonumber\\
		&\quad+D\sum\limits_{k=1}^{T}\frac{1}{{{\alpha }_{k}}}\sum_{i=1}^n\left\|x_k^*-x_{k+1}^*\right\|+\hspace{-2pt}\sum\limits_{k=1}^{T}\hspace{-2pt}{\frac{{\alpha }_{k}}{2}\hspace{-2pt}\sum\limits_{i=1}^{n}\hspace{-2pt}{\mathbb{E}\left[ {{\left\| g_{i,k}^{s} \right\|}^{2}} \right]}}\nonumber\\ 
		& \le \frac{n{D}^{2}}{2{{\alpha }_{T}}}+\sum\limits_{k=1}^{T}\frac{1}{2{{\alpha }_{k}}}\sum_{i=1}^n\left\|x_k^*-x_{k+1}^*\right\|^2\nonumber\\
	&\quad+D\sum\limits_{k=1}^{T}\frac{1}{{{\alpha }_{k}}}\sum_{i=1}^n\left\|x_k^*-x_{k+1}^*\right\|+\hspace{-2pt}\sum\limits_{k=1}^{T}\hspace{-2pt}{\frac{{\alpha }_{k}}{2}\hspace{-2pt}\sum\limits_{i=1}^{n}\hspace{-2pt}{\mathbb{E}\left[ {{\left\| g_{i,k}^{s} \right\|}^{2}} \right]}}
	\end{align}
	where the last inequality is based on the following relation:
	\begin{align}
		& \sum\limits_{k=1}^{T}\frac{\mathbb{E}[\Lambda_{k}]-\mathbb{E}[\Lambda_{k+1}]}{2{{\alpha }_{k}}} \nonumber\\ 
		& = \frac{\mathbb{E}[\Lambda_{1}]}{2{{\alpha }_{1}}}- \frac{\mathbb{E}[\Lambda_{T+1}]}{2{{\alpha }_{T}}}+ \sum\limits_{k=2}^{T} \left[ \left( \frac{1}{2{{\alpha }_{k}}} - \frac{1}{2{{\alpha }_{k-1}}} \right)\mathbb{E}[\Lambda_{k}] \right] \nonumber\\ 
		& \le \frac{n{{D}^{2}}}{2{{\alpha }_{1}}} + \sum\limits_{k=2}^{T} \left( \frac{1}{2{{\alpha }_{k}}} - \frac{1}{2{{\alpha }_{k-1}}} \right) n{{D}^{2}} \le \frac{n{{D}^{2}}}{2{{\alpha }_{T}}}. \nonumber
	\end{align}
	\par
   We turn our attention to the left-hand side of \eqref{th1}. From Assumptions 3 and 4, it follows that
	\begin{align}
		& \sum\limits_{k=1}^{T}{\sum\limits_{i=1}^{n}{\mathbb{E}\left[ \left\langle {{x}_{i,k}}-{{x}_{k}^*},\nabla f_{i,k}^{s}\left( {{x}_{i,k}} \right) \right\rangle  \right]}} \nonumber\\ 
		& =\sum\limits_{k=1}^{T}{\sum\limits_{i=1}^{n}{\mathbb{E}\left[ \left\langle {{x}_{i,k}}-{{x}_{j,k}},\nabla f_{i,k}^{s}\left( {{x}_{j,k}} \right) \right\rangle  \right]}} \nonumber\\
		& \quad+\sum\limits_{k=1}^{T}{\sum\limits_{i=1}^{n}{\mathbb{E}\left[ \left\langle {{x}_{j,k}}-{{x}_{k}^*},\nabla f_{i,k}^{s}\left( {{x}_{j,k}} \right) \right\rangle  \right]}} \nonumber\\  
		& \quad+\sum\limits_{k=1}^{T}{\sum\limits_{i=1}^{n}{\mathbb{E}\left[ \left\langle {{x}_{i,k}}-{{x}_{k}^*},\nabla f_{i,k}^{s}\left( {{x}_{i,k}} \right)-\nabla f_{i,k}^{s}\left( {{x}_{j,k}} \right) \right\rangle  \right]}} \nonumber\\ 
		& \ge \sum\limits_{k=1}^{T}{\sum\limits_{i=1}^{n}{\mathbb{E}\left[ \left\langle {{x}_{j,k}}-{{x}_{k}^*},\nabla f_{i,k}^{s}\left( {{x}_{j,k}} \right) \right\rangle  \right]}}\nonumber\\
		&\quad-(L_0+L_1D)\sum\limits_{k=1}^{T}{\sum\limits_{i=1}^{n}{\mathbb{E}\left[ \left\| {{x}_{i,k}}-{{x}_{j,k}} \right\| \right]}}\text{.} \nonumber
	\end{align}
	This, together with \eqref{th1}, gives
	\begin{align}
		& \sum\limits_{k=1}^{T}{\sum\limits_{i=1}^{n}{\mathbb{E}\left[ \left\langle {{x}_{j,k}}-{{x}_{k}^*},\nabla f_{i,k}^{s}\left( {{x}_{j,k}} \right) \right\rangle  \right]}} \nonumber\\ 
		& \le \frac{n{{D}^{2}}}{2{{\alpha }_{T}}}+\sum\limits_{k=1}^{T}{\frac{{{\alpha }_{k}}}{2}\sum\limits_{i=1}^{n}{\mathbb{E}\left[ {{\left\| g_{i,k}^{s} \right\|}^{2}} \right]}}\nonumber\\
		&\quad+(L_0+L_1D)\sum\limits_{k=1}^{T}{\sum\limits_{i=1}^{n}{\mathbb{E}\left[ \left\| {{x}_{i,k}}-{{x}_{j,k}} \right\| \right]}}\nonumber \\
		&\quad+\sum\limits_{k=1}^{T}\frac{1}{2{{\alpha }_{k}}}\sum_{i=1}^n\left\|x_k^*\hspace{-2pt}-\hspace{-2pt}x_{k+1}^*\right\|^2\hspace{-2pt}+\hspace{-2pt}D\sum\limits_{k=1}^{T}\frac{1}{{{\alpha }_{k}}}\sum_{i=1}^n\left\|x_k^*\hspace{-2pt}-\hspace{-2pt}x_{k+1}^*\right\|. \nonumber
	\end{align}
	According to Lemma 1,
	\begin{align}
		& \sum\limits_{k=1}^{T}{\sum\limits_{i=1}^{n}{\mathbb{E}\left[ \left\langle {{x}_{j,k}}-{{x}_{k}},\nabla f_{i,k}\left( {{x}_{j,k}} \right) \right\rangle  \right]}} \nonumber\\ 
		& =\sum\limits_{k=1}^{T}{\sum\limits_{i=1}^{n}{\mathbb{E}\left[ \left\langle {{x}_{j,k}}-{{x}_{k}},\nabla {{f}_{i,k}}\left( {{x}_{j,k}} \right) \right\rangle  \right]}}\nonumber\\  
		&\quad +\sum\limits_{k=1}^{T}{\sum\limits_{i=1}^{n}{\mathbb{E}\left[ \left\langle {{x}_{j,k}}-{{x}_{k}},\nabla f_{i,k}\left( {{x}_{j,k}} \right)-\nabla {{f}_{i,k}^{s}}\left( {{x}_{j,k}} \right) \right\rangle  \right]}} \nonumber\\ 
		& \le \sum\limits_{k=1}^{T}\hspace{-2pt}{\sum\limits_{i=1}^{n}\hspace{-2pt}{\mathbb{E}\left[ \left\langle {{x}_{j,k}}\hspace{-2pt}-\hspace{-2pt}{{x}_{k}},\nabla {{f}_{i,k}^{s}}\left( {{x}_{j,k}} \right) \right\rangle  \right]}}\hspace{-2pt}+\hspace{-2pt}nD{{\left( d\hspace{-2pt}+\hspace{-2pt}3 \right)}^{3/2}}\hspace{-2pt}\sum\limits_{k=1}^{T}\hspace{-2pt}{{{\mu }_{k}}}\text{.} \nonumber 
	\end{align}
	By integrating the above two inequalities and the findings from Lemmas 1--3, we have
	\begin{align}\label{thmm1}
		& \mathbb{E}-\mathcal{D}\mathcal{R}_{j,T}^\text{ncv} \nonumber\\ 
		& \le \frac{nD^{2}}{{2{\alpha }_{T}}}+nD{{\left( d+3 \right)}^{3/2}}\sum\limits_{k=1}^{T}{{{\mu }_{k}}}+\mathcal{H}_1\sum\limits_{k=1}^T\alpha_k\nonumber\\
		&\quad+\mathcal{H}_2\sum_{k=1}^T\frac{\alpha_{k}}{\mu_k^2} \sum\limits_{i=1}^{n}{\tilde{\Theta}_{i,k}^2}+ \mathcal{H}_3 \sum_{k=1}^T\frac{\alpha_{k}}{\mu_k} \sum\limits_{i=1}^{n}{\tilde{\Theta}_{i,k}}\nonumber\\
		&\quad+\frac{n\omega_T^2}{2\alpha_{T}}+\frac{nD\omega_T}{\alpha_{T}}\text{,}
	\end{align}
	where $\mathcal{H}_1=48\left( d+4\right)^2+\frac{\sqrt{3}(32+24n\Gamma )\left( d+4\right)(L_0+L_1D) }{1-\gamma}$, $\mathcal{H}_2=48nd$, $\mathcal{H}_3=\frac{(32+24n\Gamma )(L_0+L_1D)\sqrt{3d}}{1-\gamma}$, $\tilde{\Theta}_{i,k}=\max_{\tau\in[k]}\theta_{i,\tau}$, and $\theta_{i,k}$, $\pi_{i,k}$ are defined in \eqref{theta_T} and \eqref{omega_k}, respectively. Substituting the explicit expressions \(\alpha_k = \frac{1}{2Mk^a}\), \(\mu_k = \frac{1}{k^a}\) where \(a \le b\), and \(M = \frac{\sqrt{3d}(8 + 6n\Gamma)L_0}{\gamma(1 - \gamma)}\), we can derive:

	\begin{align}
	& \sum\limits_{k=1}^{T}{\sum\limits_{i=1}^{n}{\mathbb{E}\left[ \left\langle {{x}_{j,k}}-{{x}_{k}},\nabla f_{i,k}\left( {{x}_{j,k}} \right) \right\rangle  \right]}}\nonumber\\ 
	&\le nD{{\left( d+3 \right)}^{3/2}}T^{1-b}+ \frac{n{D}^{2}\gamma\left( 1-\gamma\right) }{\sqrt{3d}(8+6n\Gamma){{L}_{0}}}T^a\nonumber\\
	&+\hspace{-2pt}\left\lbrace \frac{4\sqrt{3}\left( d\hspace{-2pt}+\hspace{-2pt}4\right)^2\gamma(1\hspace{-2pt}-\hspace{-2pt}\gamma)}{\sqrt{d}(4\hspace{-2pt}+\hspace{-2pt}3n\Gamma)L_0}\hspace{-2pt}+\hspace{-2pt}\frac{2(L_0\hspace{-2pt}+\hspace{-2pt}L_1D)\left( d\hspace{-2pt}+\hspace{-2pt}4\right)\gamma}{\sqrt{d}L_0}\right\rbrace T^{1-a}\nonumber\\
	&+\hspace{-2pt}\frac{4\sqrt{3d}n\gamma(1\hspace{-2pt}-\hspace{-2pt}\gamma)}{(4+3n\Gamma)L_0}{T}^{2b-a}\Theta_{T}^2\hspace{-2pt}+\hspace{-2pt}\left( 2\gamma\hspace{-2pt}+\hspace{-2pt}\frac{2\gamma L_1D}{L_0}\right) {T}^{b-a}\Theta_{T}\nonumber \\
	&+\frac{\sqrt{3d}(8 + 6n\Gamma)nL_0}{\gamma(1 - \gamma)}\left({T^a\omega_T^2}+{2DT^a\omega_T} \right) 
	\text{.}
\end{align}
Furthermore, if \(a = \frac{1}{2}\) and $a-b\ge\delta\ge0$, then \(T^{\max\{a,1-a, 1 - b\}} \le T^{1/2 + \delta}\), \(T^{2b - a} \Theta_T^2 \le T^{1/2 - 2\delta}\Theta_T^2 \), \(T^{b - a} \le T^{-\delta} \Theta_T\), $T^a\omega_T^2=T^{1/2}\omega_T^2 $, and $T^a\omega_T=T^{1/2}\omega_T $, which leads to the result in \eqref{thm2}. In particular, when \(a = b = \frac{1}{2}\), the dynamic regret is \(\mathcal{O}(T^{1/2}+ T^{1/2}\Theta_T^2+T^{1/2}\omega_T^2 )\). The proof is complete.
	\end{proof} 
\subsection{Proof of Theorem 2}
\begin{proof}
When convex function $f_{i,k}\in C^{0,0}$ with constant $L_0$, by using the convexity of $f_{i,k}$ and considering \eqref{th1}, we can obtain
\begin{align}
	& \sum\limits_{k=1}^{T}{\sum\limits_{i=1}^{n}{\mathbb{E}\left[f_{i,k}^{s}\left( {{x}_{i,k}} \right)-f_{i,k}^{s}\left( x_k^* \right)\right]}}\nonumber\\
	& \le \frac{n{D}^{2}}{2{{\alpha }_{T}}}+\sum\limits_{k=1}^{T}\frac{1}{2{{\alpha }_{k}}}\sum_{i=1}^n\left\|x_k^*-x_{k+1}^*\right\|^2\nonumber\\
	&\quad+D\sum\limits_{k=1}^{T}\frac{1}{{{\alpha }_{k}}}\sum_{i=1}^n\left\|x_k^*-x_{k+1}^*\right\|+\hspace{-2pt}\sum\limits_{k=1}^{T}\hspace{-2pt}{\frac{{\alpha }_{k}}{2}\hspace{-2pt}\sum\limits_{i=1}^{n}\hspace{-2pt}{\mathbb{E}\left[ {{\left\| g_{i,k}^{s} \right\|}^{2}} \right]}}.\nonumber
\end{align}
By adding and subtracting the term $f_{i,k}^{s}\left( x_{j,k}\right)$ with any $j\in\mathcal{V}$ and using the Lipschitz continuity of $f_{i,k}$, it follows that
\begin{align}
	& \sum\limits_{k=1}^{T}{\sum\limits_{i=1}^{n}{\mathbb{E}\left[f_{i,k}^{s}\left( {{x}_{j,k}} \right)-f_{i,k}^{s}\left( x_k^* \right)\right]}}\nonumber\\
	& \le \frac{n{D}^{2}}{2{{\alpha }_{T}}}+\sum\limits_{k=1}^{T}\frac{1}{2{{\alpha }_{k}}}\sum_{i=1}^n\left\|x_k^*-x_{k+1}^*\right\|^2\nonumber\\
	&\quad+L_0 \sum\limits_{k=1}^{T}\sum\limits_{i=1}^{n}\mathbb{E}\left[ \left\|x_{i,k}\hspace{-2pt}-\hspace{-2pt}x_{j,k} \right\| \right]\hspace{-2pt}+\hspace{-2pt}D\sum\limits_{k=1}^{T}\frac{1}{{{\alpha }_{k}}}\sum_{i=1}^n\left\|x_k^*\hspace{-2pt}-\hspace{-2pt}x_{k+1}^*\right\|\nonumber\\
	&\quad+\sum\limits_{k=1}^{T}\hspace{-2pt}{\frac{{\alpha }_{k}}{2}\sum\limits_{i=1}^{n}{\mathbb{E}\left[ {{\left\| g_{i,k}^{s} \right\|}^{2}} \right]}}.\nonumber
\end{align}
We then combine the preceding inequality and the results in Lemmas 2 and 5 to get
\begin{align}
	 &\mathbb{E}-\mathcal{D}\mathcal{R}_{j,T}^\text{cv}=\sum\limits_{k=1}^{T}{\sum\limits_{i=1}^{n}{\mathbb{E}\left[f_{i,k}^{s}\left( {{x}_{j,k}} \right)-f_{i,k}^{s}\left( x_k^* \right)\right]}} \nonumber\\ 
	& \le \frac{nD^{2}}{{2{\alpha }_{T}}}+2nD{{\left( d+3 \right)}^{3/2}}\sum\limits_{k=1}^{T}{{{\mu }_{k}}}+\mathcal{H}_1\sum\limits_{k=1}^T\alpha_k+\frac{n\omega_T^2}{2\alpha_{T}}\nonumber\\
	&\quad+\mathcal{H}_2\sum_{k=1}^T\frac{\alpha_{k}}{\mu_k^2} \sum\limits_{i=1}^{n}{\tilde{\Theta}_{i,k}^2}+ \mathcal{H}_3 \sum_{k=1}^T\frac{\alpha_{k}}{\mu_k} \sum\limits_{i=1}^{n}{\tilde{\Theta}_{i,k}}+\frac{nD\omega_T}{\alpha_{T}}\text{,}\nonumber
\end{align}
where $\mathcal{H}_1=48\left( d+4\right)^2+\frac{\sqrt{3}(32+24n\Gamma )\left( d+4\right)L_0 }{1-\gamma}$, $\mathcal{H}_2=48nd$, $\mathcal{H}_3=\frac{(32+24n\Gamma )L_0\sqrt{3d}}{1-\gamma}$, $\tilde{\Theta}_{i,k}=\max_{\tau\in[k]}\theta_{i,k}$, and $\theta_{i,k}$, $\omega_{T}$ are defined in \eqref{theta_T} and \eqref{omega_k}, respectively. Substituting the explicit expressions \(\alpha_k = \frac{1}{2M\sqrt{k}}\), \(\mu_k = \frac{1}{\sqrt{k}}\), and \(M = \frac{\sqrt{3d}(8 + 6n\Gamma)L_0}{\gamma(1 - \gamma)}\), we can derive the desired result.\par 

\end{proof}
\section*{References}
\bibliographystyle{IEEEtran}
\bibliography{refff}

\begin{IEEEbiography}[{\includegraphics[width=1in,height=1.25in,clip,keepaspectratio]{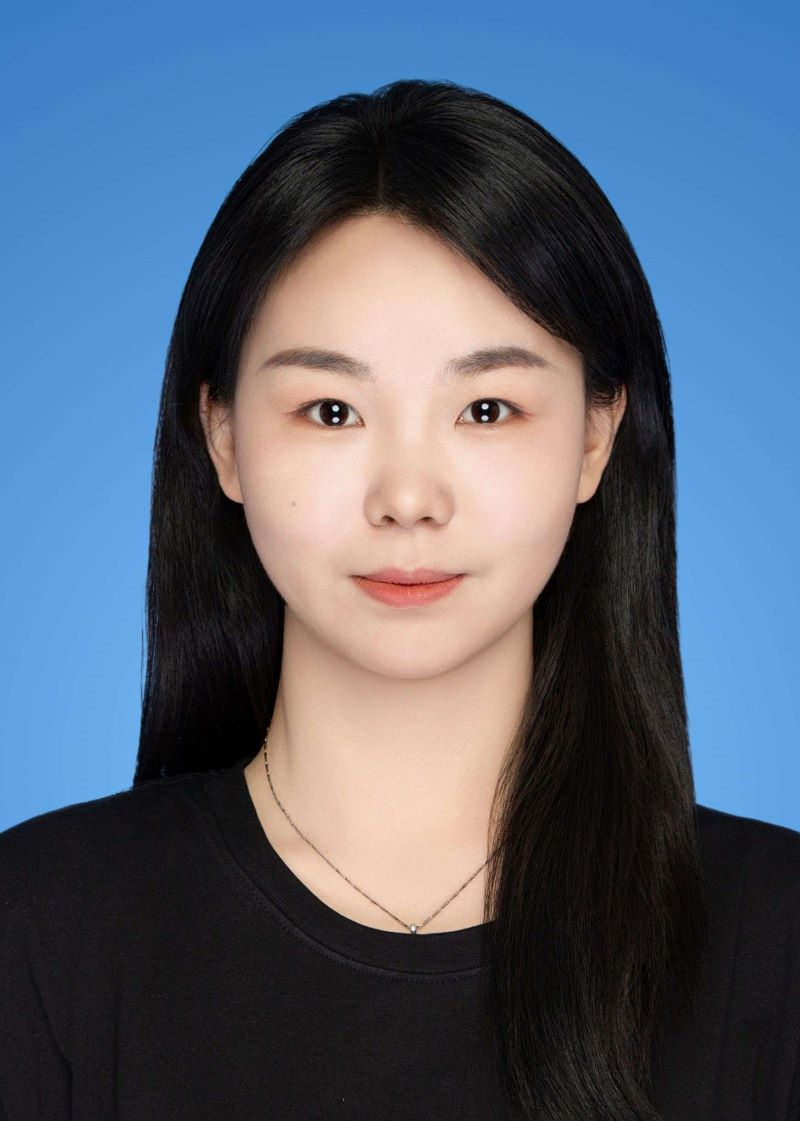}}]{Youqing Hua} received the B.S. degree from the School of Mathematical Sciences, Qufu normal University, Qufu, China, in 2022. She is currently pursuing a M.S. in School of Control Science and Engineering, Shandong University, Jinan, China.\par 
Her research focuses on distributed optimization and multiagent systems.
\end{IEEEbiography}

\begin{IEEEbiography}[{\includegraphics[width=1in,height=1.25in,clip,keepaspectratio]{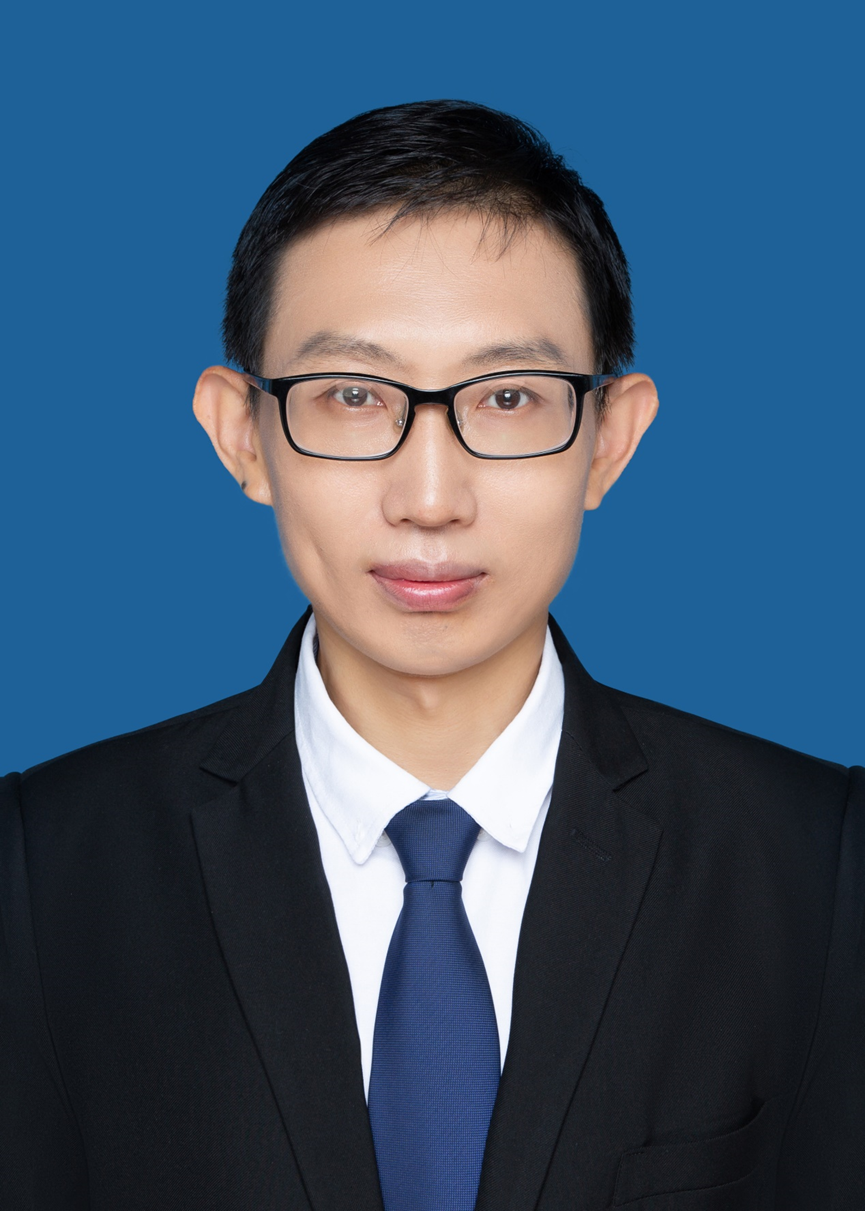}}]{Shuai Liu} (IEEE Member) received the B.E. and M.E. degrees in control science and engineering from Shandong University, Jinan, China, in 2004 and 2007, respectively, and the Ph.D. degree in electrical and electronic engineering from Nanyang Technological University, Singapore, in 2012.\par
	From 2011 to 2017, he was a Senior Research Fellow with Berkeley Education Alliance, Singapore. Since 2017, he has been with the School of Control Science and Engineering, Shandong University. His research interests include distributed control, estimation and optimization, smart grid, integrated energy system, and machine learning. 
\end{IEEEbiography}

\begin{IEEEbiography}[{\includegraphics[width=1in,height=1.25in,clip,keepaspectratio]{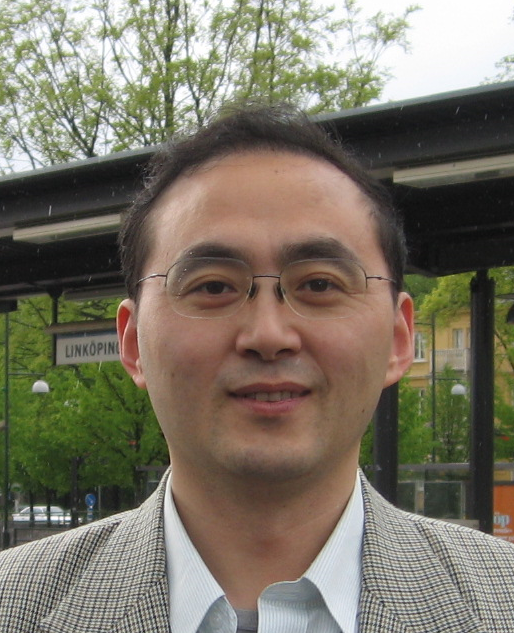}}]{Yiguang Hong} (Fellow, IEEE) received the B.S. and M.S. degrees from Peking University, Beijing, China, in 1987 and 1990, respectively, and the Ph.D. degree from the Chinese Academy of Sciences (CAS), Beijing, China, in 1993.\par 
He is currently a Professor with the Department of Control Science and Engineering and Shanghai Research Institute for Intelligent Autonomous Systems, Tongji University, Shanghai, China. Before October 2020, he was a Professor with the Academy of Mathematics and Systems Science, CAS.\par  
His current research interests include nonlinear control, multiagent systems, distributed optimization and game, machine learning, and social networks.
\end{IEEEbiography}

\begin{IEEEbiography}[{\includegraphics[width=1in,height=1.25in,clip,keepaspectratio]{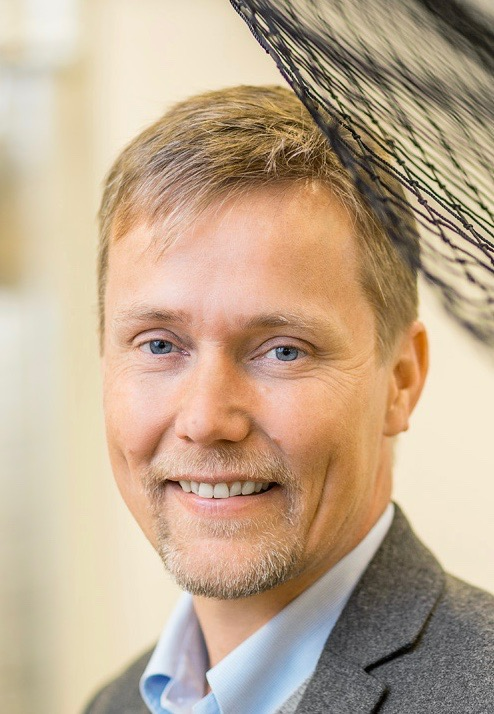}}]{Karl Henrik Johansson} (Fellow, IEEE) received the M.Sc. and Ph.D. degrees in electrical engineering from Lund University, Lund, Sweden, in 1992 and 1997, respectively.\par 
He is currently a Professor with the School of Electrical Engineering and Computer Science, KTH Royal Institute of Technology, Stockholm, Sweden. He has held visiting positions with University of California, Berkeley, Berkeley, CA, USA, Caltech, Pasadena, CA, USA, Nanyang Technological University, Singapore, HKUST Institute of Advanced Studies, Hong Kong, and Norwegian University of Science and Technology, Trondheim, Norway. His research interests include networked control systems, cyber-physical systems, and applications in transportation, energy, and automation.
\end{IEEEbiography}

\begin{IEEEbiography}[{\includegraphics[width=1in,height=1.25in,clip,keepaspectratio]{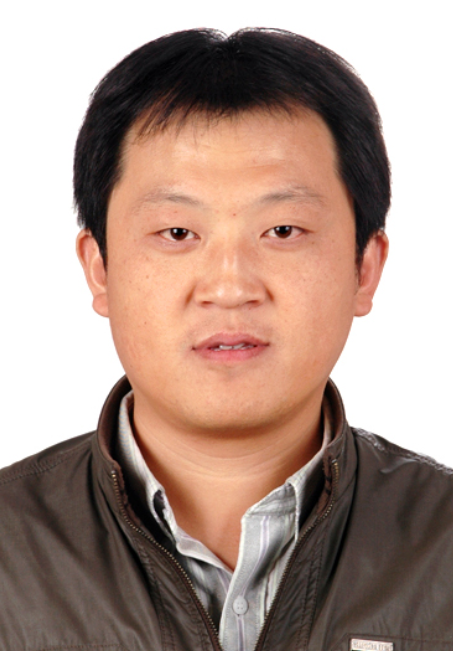}}]{Guangchen Wang} (Senior Member, IEEE) received the B.S. degree in mathematics from Shandong Normal University, Jinan, China, in 2001, and the Ph.D. degree in probability theory and mathematical statistics from the School of Mathematics and System Sciences, Shandong University, Jinan, China, in 2007.\par
From July 2007 to August 2010, he was a Lecturer with the School of Mathematical Sciences, Shandong Normal University. In September 2010, he joined as an Associate Professor the School of Control Science and Engineering, Shandong University, where he has been a Full Professor since September 2014. His current research interests include stochastic control, nonlinear filtering, and mathematical finance.
\end{IEEEbiography}

\end{document}